\documentclass[twoside,11pt]{article}

\usepackage[T1]{fontenc}
\usepackage[utf8]{inputenc}
\usepackage{amsmath}
\usepackage{jmlr2e}

\usepackage{booktabs}
\usepackage{array}
\usepackage{longtable}   
\usepackage{microtype}
\usepackage{xcolor}
\usepackage{tikz}
\usetikzlibrary{arrows.meta,positioning,fit,backgrounds,shapes.geometric,calc,patterns,patterns.meta}
\usepackage{pgfplots}
\pgfplotsset{compat=1.16}

\definecolor{idblue}{RGB}{31,119,180}
\definecolor{oodorange}{RGB}{214,96,42}
\definecolor{leakred}{RGB}{200,30,45}
\definecolor{okgreen}{RGB}{34,139,76}
\definecolor{panelgray}{RGB}{120,120,120}
\definecolor{oracleblue}{RGB}{29,82,160}
\definecolor{unsupred}{RGB}{220,90,30}
\definecolor{gapgray}{RGB}{170,170,170}
\definecolor{leakpurple}{RGB}{140,60,160}

\newtheorem{observation}[theorem]{Observation}
\renewenvironment{proof}[1][Proof]{\par\noindent{\bf #1\ }}{\hfill\BlackBox\\[2mm]}

\newcommand{\prs}{\textsc{PRS}}
\newcommand{\ahss}{\textsc{AHSS}}

\newcommand{\auroc}{AUROC}
\newcommand{\oraclelda}{OracleLDA}
\newcommand{\R}{\mathbb{R}}
\newcommand{\E}{\mathbb{E}}
\newcommand{\norm}[1]{\left\lVert#1\right\rVert}
\newcommand{\Sig}{\boldsymbol{\Sigma}}

\usepackage{lastpage}
\jmlrheading{}{2026}{1-\pageref{LastPage}}{6/26}{}{}{Vishnu Bindu Balachandran}
\ShortHeadings{Decodable but Not Detectable}{Vishnu Bindu Balachandran}
\firstpageno{1}

\begin{document}

\title{Decodable but Not Detectable:\\
A Leakage Fingerprint for Near-OOD Benchmarks}

\author{\name Vishnu Bindu Balachandran \email vishnubindubalachandran@outlook.com \\
       \addr Independent AI Researcher\\
       \url{www.vishnubindubalachandran.com}}

\editor{}

\maketitle

\begin{abstract}%
While auditing a perturbation-based out-of-distribution (OOD) detector on a
standard document benchmark, we recorded an \auroc{} of $0.326$, well under the
$0.5$ chance level; taken at face value this reads as a catastrophic method
failure, but it is instead a benchmark leak. The near-OOD split was built by
designating as ``OOD'' a class the model was trained on, an easy mistake that
leaves its training examples inside the in-distribution set used to fit the
unsupervised detector, so the detector is asked to flag a familiar,
in-distribution population as anomalous. Deleting the class and re-training $35$
models across two domains raises the same number to $0.911$, and the
below-chance behavior disappears for every unsupervised detector we tried. We
distill the contamination into a leak fingerprint, near-perfect
supervised decodability ($\auroc\!\approx\!1$) coupled with unsupervised
detection that collapses well below its clean and far-OOD level
(operationally, best-unsupervised $\auroc<0.65$; far below chance in the
original document case), and show it separates a leaked class from a genuinely hard
novel class, which is only imperfectly decodable. We
validate this fingerprint on a controlled battery of $20$ deliberately-leaked
settings and $32$ clean controls ($20$ matched, $12$ novel-source) across the full $2\times2$ of ResNet-50 and ViT-B/16
on CIFAR-10/100 (four fine-tuned backbones). In the model's
penultimate-embedding space it attains sensitivity $18/20$ and specificity $31/32$;
the decisive signal is a matched contrast: the same class trips the fingerprint
when leaked into the fit set and not when excluded from it (perfect $20/20$ on the matched
controls). The firing direction also transfers to the deployed perturbation signature, while
this cheap no-retrain contrast is reliable in the embedding space. Finally we run the audit
the test enables: across $24$ standard near/far OOD benchmark pairs on four backbones
the fingerprint fires on exactly one (the intrinsically-hard CIFAR-100$\to$CIFAR-10
near pair our scope analysis predicts) and on no far-OOD pair, confirming the
diagnostic is specific in the wild and that standard cross-dataset construction does not
introduce the fit-set leak. We then release a drop-in leave-class-out correction and
re-evaluate under it. The
re-evaluation surfaces a second, broader finding that the leak had obscured:
across $19$ clean settings, three architectures, and seven datasets, the
perturbation signals are decodable but not detectable, a supervised reader
recovers the OOD signal nearly perfectly while no unsupervised detector does, and
the perturbation-based detector does not meaningfully improve on a plain
Mahalanobis-on-embeddings baseline. For transparency we also retract, in-paper, an earlier headline
correlation we found to be an artifact of a circular measurement. The
contributions are a corrected protocol and a validated detector for the
contamination, not a new OOD method.
\end{abstract}


\section{Introduction}
\label{sec:intro}
Near-OOD detection, flagging inputs from classes close to, but not in, the
training distribution, is among the most consequential settings for OOD methods,
and among the hardest to evaluate. Evaluations typically designate one or more
held-out classes as ``OOD'' and measure how well an unsupervised detector,
fit on in-distribution (ID) data only, separates them. The validity of such
an evaluation rests on a quiet assumption: that the ``OOD'' class is genuinely
absent from both training and the ID fit set. We show this assumption is easy to
violate when constructing a near-OOD split, with consequences severe enough
to invert the sign of reported results, and we give a way to detect the failure.

We found the contamination by accident. Auditing a perturbation-response OOD
detector (\prs, Section~\ref{sec:vehicle}) on a standard document benchmark, we measured
an \auroc{} of $0.326$, far below the $0.5$ of random guessing. A below-chance
detector is a strong signal that something is structurally wrong, because it means
the ID and ``OOD'' populations are systematically reversed in the score. The
cause was a leak in how this near-OOD split was constructed: the class
designated ``OOD'' is one the model was trained on, and its training documents are
inside the one-class fit set. The detector, which scores points by distance from
the ID bulk, correctly judges this trained class to be in-distribution, and
is penalized for it. Deleting the class from training and re-fitting raises the same
number to $0.911$ (Section~\ref{sec:leak}).

This paper makes three contributions.
\begin{itemize}
  \item \textbf{A leak in near-OOD benchmark construction} (Section~\ref{sec:leak}). We
  show that designating a semantically similar trained class as ``OOD'', a
  tempting way to build a hard near-OOD split, places that class inside the ID fit
  set, manufacturing spurious below-chance ``failures'' ($0.326$) that vanish under
  a corrected leave-class-out protocol ($0.911$ on the same class),
  validated with $35$ fresh fine-tunes across two domains (documents and text).
  \item \textbf{A test for the leak, validated} (Section~\ref{sec:leak},
  Remark~\ref{rem:fingerprint}; Section~\ref{sec:validate}). We derive a leak fingerprint (near-perfect supervised decodability with unsupervised detection
  collapsed well below its clean level, best-unsup $<0.65$) that flags a
  contaminated near-OOD split without backbone
  re-training, and we validate it on a controlled battery: across $20$ deliberately-leaked settings and $32$ clean controls ($20$ matched, $12$ novel-source; ResNet-50 and ViT-B/16;
  CIFAR-10/100; four fine-tuned backbones, the full $2\times2$) it attains, in the
  model's penultimate-embedding space, sensitivity $18/20$ and specificity $31/32$
  (matched fit-set-exclusion controls perfect at $20/20$), isolating fit-set membership as the
  cause and separating a leak from a genuinely hard novel class. We then run the
  in-the-wild audit the test enables across $24$ standard near/far OOD benchmark pairs
  (Section~\ref{sec:audit}): it fires on exactly the one intrinsically-hard pair our scope
  predicts and on no far-OOD pair, so standard cross-dataset construction does not introduce
  the leak and the diagnostic is specific in the wild.
  \item \textbf{A corrected protocol and a clean re-evaluation} (Section~\ref{sec:gap}).
  Under the correction, the perturbation signals are decodable but not
  detectable: across $19$ clean settings, three architectures, and seven datasets,
  a supervised reader recovers the OOD signal ($\auroc\,0.87$--$1.00$) while no
  unsupervised detector closes the gap, and the perturbation detector does not
  meaningfully improve on plain embedding distance. This gap is consistent with,
  and measured independently of, concurrent analyses of OOD misspecification
  (Section~\ref{sec:related}); we give a short mechanism (Section~\ref{sec:theory}) for why.
\end{itemize}

We keep the perturbation detector (\prs) in the paper only as the vehicle
that exposed the leak and as the object of the clean re-evaluation; it is not
proposed as a method, and we show it is no better than a plain baseline and is
often beaten by an off-the-shelf Isolation Forest.

\section{Related Work}
\label{sec:related}
The leak and the decodability gap we report each connect to a distinct thread of prior work, from post-hoc OOD scoring to documented benchmark contamination and the broader literature on data leakage.

\subsection{Post-Hoc OOD Detection}
A large family of methods scores OOD-ness from a frozen classifier's outputs or
features. Output-based scores include maximum softmax probability \citep{msp},
temperature and input scaling \citep{odin}, free energy \citep{energy}, and
activation rectification \citep{react}. Distance- and density-based scores
instead operate in feature space: class-conditional Gaussian (Mahalanobis)
scoring \citep{maha} and non-parametric $k$-nearest-neighbor distance \citep{knn}.
Our study uses these last two as the baselines the perturbation detector must
beat, and finds that it does not.

\subsection{Near-OOD Detection and Its Evaluation}
Near-OOD, where the held-out classes are semantically close to the training
distribution, is the hardest and most safety-relevant regime, and the one our
leak concerns. \citet{ren} show that plain Mahalanobis degrades on
near-OOD and propose a relative correction, and OpenOOD \citep{openood}
standardizes near- and far-OOD benchmarks for comparison. We contribute a
different observation, about evaluation construction rather than the
detector: a near-OOD split built by designating a trained class as ``OOD'' places
that class inside the in-distribution fit set, so a reported near-OOD score can
reflect contamination rather than detector quality. The attention-masking method
we audit \citep{base} in fact constructs its intra-dataset OOD correctly, holding
the designated class out of training; the leak we study arises precisely when that
discipline is not followed.

\subsection{Contamination and Ill-Posedness in OOD/Anomaly Benchmarks}
Our finding relates directly to two lines of work. At ImageNet scale, \citet{ninco}
show that widely-used OOD test sets are contaminated with
in-distribution content (some contain more than half images whose objects
belong to ID classes), and release the hand-cleaned NINCO benchmark to repair the
resulting distortion of detector rankings. Their contamination sits in the
OOD test images (ID objects mislabeled as OOD, which inflates false
positives); ours sits in the fit set (a trained ID class designated
``OOD'' while its training data stays inside the one-class fit), which inverts
the score rather than merely inflating it. Closer still, and concurrent, \citet{testingthetest}
show that class-split anomaly protocols become
ill-posed when the held-out class overlaps the normal mixture in
representation space, causing scores to collapse or invert, and propose a
training-free ``neighborhood class leakage'' diagnostic validated on Fashion-MNIST,
CIFAR-10, and Imagenette. We reach the same core phenomenon independently (leakage
makes a held-out class score below an honest baseline, detectable without
re-training) but from a different starting point (a real, deployed
multimodal-document detector, not a constructed toy split) and with a
different diagnostic: where their index measures unsupervised
representation overlap, our fingerprint contrasts near-perfect supervised
decodability ($\oraclelda\!\approx\!1$) against unsupervised collapse, the ``decodable but not detectable'' signature (Section~\ref{sec:gap}). In addition, it separates a genuine leak from an intrinsically hard but clean near-OOD
pair (Section~\ref{sec:validate}). We read their result as strong independent
corroboration that benchmark-construction leakage is a real and general hazard; our
complementary contributions are the supervised-decodability fingerprint, its
discovery in a deployed benchmark, the corrected leave-class-out protocol, and the
controlled sensitivity/specificity validation. We make the complementarity
empirical: a faithful reimplementation of their k-NN class-leakage index, run on our
$52$-setting battery, separates leak from clean at AUROC $0.52$ (chance), with near-identical mean leakage on leaked ($0.087$) and clean ($0.094$) settings, whereas our
fingerprint reaches AUROC $0.995$ (Appendix~\ref{app:headtohead}). The reason is
mechanistic and is the crux of our delta: their index detects geometric overlap (a
held-out class mixed into the ID neighborhood), but our leaked classes are separable
clusters that merely sit inside the one-class fit set (low neighborhood leakage), so
an overlap index is blind to them. The two diagnose distinct, complementary contamination
mechanisms; a benchmark auditor needs both.

\subsection{Is Supervised OOD Detection Misspecified?}
A concurrent position paper \citep{wrongquestions} argues that supervised OOD
detectors are fundamentally misspecified: a model trained only on ID
classes answers ``are these features atypical?'' rather than ``is this point from a
different distribution?'', leaving an irreducible gap between any ID-only
detector and an oracle with OOD access. Our contribution is complementary and
operates at a different level. They diagnose a conceptual limitation of the
detection task on clean benchmarks; we diagnose a construction flaw
in a specific benchmark, a trained class placed in the ``OOD'' set, that is not
irreducible but removable by protocol (\auroc{} $0.326\!\to\!0.911$). The two
findings reinforce each other: their oracle-versus-method gap is an instance of the
same oracle-versus-unsupervised gap we measure on a different suite, and their result
that degrading a feature covariance toward the identity improves OOD
detection independently corroborates the whitening mechanism we derive in
Section~\ref{sec:theory}. Where they argue the question is wrong, we show that one
widely-used way of asking it was also contaminated.

\subsection{Perturbation- and Sensitivity-Based Detection}
A separate line reads OOD-ness from how a model responds to perturbations.
Attention-head masking \citep{base} summarizes the representation shift under random
head ablations into a scalar sensitivity score. The detector we audit (\prs) is a
vector generalization of that scalar; our negative result therefore bears directly
on the perturbation-sensitivity premise, and our mechanism (Section~\ref{sec:theory})
explains why the richer representation still does not help. The baselines and
fixed-seed protocol we inherit here were established in a prior cross-architecture
audit of attention-head-masking detection \citep{audit}, which first reported the
near-OOD degradation this paper traces to a benchmark-construction leak; the
present work isolates and corrects that leak and adds the clean re-evaluation.

\subsection{Data Leakage in ML Benchmarks}
Train/test contamination is a recognized, recurring failure mode: \citet{kaufman}
give a foundational taxonomy of leakage, and \citet{leakage}
document it across $17$ scientific fields and hundreds of papers, frequently
producing wildly optimistic conclusions. Our finding is a concrete,
mechanism-level instance in OOD evaluation, a trained class used as
``OOD'' while its examples sit inside the one-class fit set, accompanied by a
retraining-free test to detect it, complementing the checklist-style prevention
those works advocate.

\subsection{Covariance Estimation}
Our detector uses Ledoit--Wolf shrinkage \citep{ledoitwolf} for the signature
covariance; the whitening pathology we prove and measure (Section~\ref{sec:theory})
characterizes what shrinkage can and cannot repair when the discriminative signal
concentrates in high-variance directions.

\section{The Vehicle: A Perturbation-Response Detector}
\label{sec:vehicle}
We summarize only what is needed to read the results; the detector is not a
contribution. Let $f$ be a frozen fine-tuned transformer with \textsc{cls}
embedding $\phi:\mathcal X\!\to\!\R^{m}$. For a fixed bank of $K{=}25$ binary
attention-head masks $\{M_1,\dots,M_K\}$ plus the clean pass $\mathbf 1$, we run $f$
$K{+}1$ times and form a $77$-dimensional perturbation-response signature
$s(x)$ from four blocks: per-mask displacement
$d_k=\norm{\hat\phi(x;M_k)-\hat\phi(x;\mathbf 1)}^2$, norm ratios
$r_k=\norm{\phi(x;M_k)}/\norm{\phi(x;\mathbf 1)}$, distances $g_0,g_k$ to ID class
centroids, and a nearest-centroid flip rate ($\hat\phi$ denotes $L_2$
normalization). The base method's scalar is the sign-committed collapse
$\ahss(x)=-\tfrac1K\sum_k d_k(x)$\citep{base}.

The primary detector \prs-Maha fits a one-class Gaussian with
Ledoit--Wolf shrinkage on standardized ID-train signatures only
(here $D{=}77$ is the signature dimension),
\begin{equation}
s_{\mathrm{Maha}}(z) = -\,(z-\mu)^\top \Sig^{-1}(z-\mu), \qquad
\Sig=(1-\alpha)\hat\Sig+\alpha\tfrac{\mathrm{tr}\hat\Sig}{D}I,
\label{eq:maha}
\end{equation}
and scores larger $=$ more in-distribution, with the shrinkage intensity
$\alpha\!\in\![0,1]$ set by the Ledoit--Wolf estimator. Baselines Mahalanobis,
$k$NN, and \ahss run on the raw embedding. \oraclelda{} is a
diagnostic, not a detector: the best label-supervised linear (Fisher) reader
of $s(x)$, which measures whether the OOD signal is present at all. No statistic is
ever fit on OOD data (enforced by a unit test), and every inherited baseline is
reconciled to the prior audit it comes from \citep{audit} to $\Delta\!\le\!10^{-4}$, so any
anomaly is a property of the data, not an implementation bug.

\section{A Benchmark Leak in Near-OOD Evaluation}
\label{sec:leak}

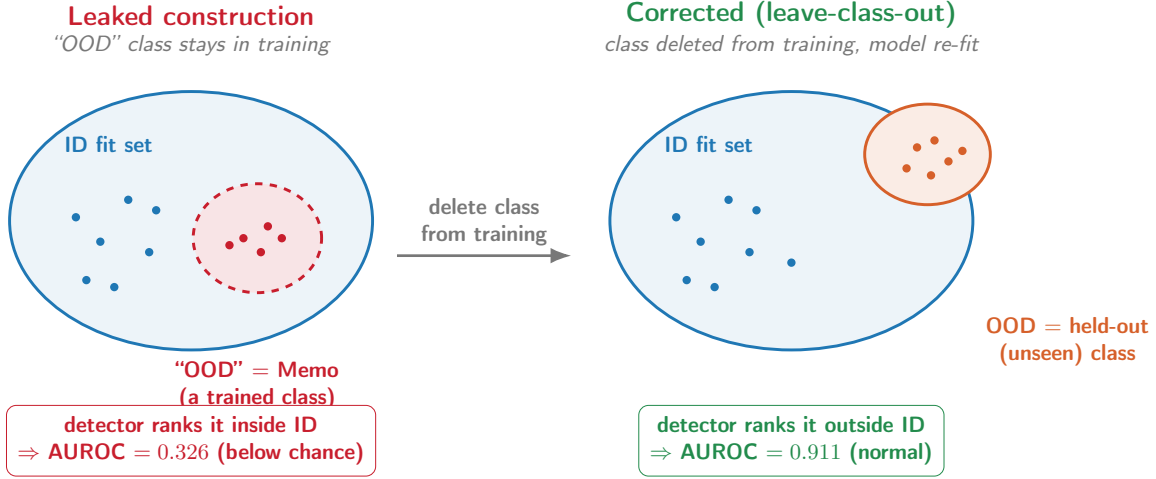
\begin{figure}[htbp]
\centering
\resizebox{\textwidth}{!}{%
\begin{tikzpicture}[font=\sffamily\small]
\begin{scope}[local bounding box=L]
  \node[font=\sffamily\bfseries, leakred] at (2.5,4.75) {Leaked construction};
  \node[font=\sffamily\itshape\footnotesize, panelgray] at (2.5,4.32) {``OOD'' class stays in training};
  \draw[idblue, very thick, fill=idblue!8] (2.5,1.8) ellipse (2.6 and 1.85);
  \node[idblue, font=\sffamily\bfseries\footnotesize, anchor=north west] at (0.55,3.15) {ID fit set};
  \foreach \x/\y in {1.2/1.5, 1.6/2.1, 1.0/0.95, 1.9/1.35, 1.4/0.85, 2.0/1.95, 0.85/1.85}
     \fill[idblue] (\x,\y) circle (1.7pt);
  \fill[leakred!12] (3.45,1.55) ellipse (0.92 and 0.78);
  \draw[leakred, very thick, dashed] (3.45,1.55) ellipse (0.92 and 0.78);
  \foreach \x/\y in {3.25/1.55, 3.6/1.72, 3.5/1.35, 3.05/1.45, 3.8/1.55} \fill[leakred] (\x,\y) circle (1.7pt);
  \node[leakred, font=\sffamily\bfseries\footnotesize, align=center] at (3.45,-0.55) {``OOD'' $=$ Memo\\(a trained class)};
  \node[draw=leakred, rounded corners, fill=white, text=leakred, font=\sffamily\bfseries\footnotesize, align=center, inner sep=4.5pt] at (2.5,-1.35)
        {detector ranks it inside ID\\$\Rightarrow$ AUROC $= 0.326$ (below chance)};
\end{scope}
\begin{scope}[xshift=8.6cm, local bounding box=R]
  \node[font=\sffamily\bfseries, okgreen] at (2.5,4.75) {Corrected (leave-class-out)};
  \node[font=\sffamily\itshape\footnotesize, panelgray] at (2.5,4.32) {class deleted from training, model re-fit};
  \draw[idblue, very thick, fill=idblue!8] (2.5,1.8) ellipse (2.6 and 1.85);
  \node[idblue, font=\sffamily\bfseries\footnotesize, anchor=north west] at (0.55,3.15) {ID fit set};
  \foreach \x/\y in {1.2/1.5, 1.6/2.1, 1.0/0.95, 1.9/1.35, 1.4/0.85, 2.0/1.95, 0.85/1.85, 2.5/1.2}
     \fill[idblue] (\x,\y) circle (1.7pt);
  \draw[oodorange, very thick, fill=oodorange!12] (4.45,2.75) ellipse (0.9 and 0.72);
  \foreach \x/\y in {4.15/2.55, 4.55/2.95, 4.75/2.65, 4.30/2.85, 4.95/2.80, 4.50/2.45} \fill[oodorange] (\x,\y) circle (1.7pt);
  \node[oodorange, font=\sffamily\bfseries\footnotesize, align=center, anchor=north west] at (5.15,0.55) {OOD $=$ held-out\\(unseen) class};
  \node[draw=okgreen, rounded corners, fill=white, text=okgreen, font=\sffamily\bfseries\footnotesize, align=center, inner sep=4.5pt] at (2.5,-1.35)
        {detector ranks it outside ID\\$\Rightarrow$ AUROC $= 0.911$ (normal)};
\end{scope}
\draw[-{Latex[length=3mm]}, line width=1.1pt, panelgray]
      (5.45,1.3) -- node[above, font=\sffamily\bfseries\footnotesize, panelgray, align=center] {delete class\\from training} ++(2.5,0);
\end{tikzpicture}%
}
\caption{Why the leak inverts the score. \textbf{Left:} when a semantically similar
trained class (Memo) is designated ``OOD'', its examples lie inside
the one-class ID fit set, so an unsupervised detector correctly ranks them as
in-distribution and is penalized, scoring below chance ($0.326$). \textbf{Right:}
removing the class from training (leave-class-out) places a genuinely unseen class
at the edge of the fit set (near-OOD: semantically close but not trained),
and the same detector scores normally ($0.911$). Only
the unsupervised score is affected by the leak; a supervised reader, which
sees labels, is not (it separates the trained class perfectly, AUROC $\approx 1.0$,
and a genuinely novel class only imperfectly, $\approx 0.88$, a contrast we use as
a leak fingerprint in Section~\ref{sec:leak}).}
\label{fig:leak}
\end{figure}

\subsection{The Symptom}
On a near-OOD split of the Tobacco document benchmark, in-distribution classes
vs.\ a semantically similar class designated ``OOD'', \prs-Maha scores
\auroc{}~$0.326$, and the embedding baselines are also depressed (Mahalanobis
$0.555$, $k$NN $0.634$, both far below their clean level and \prs-Maha actually
below chance). Yet the supervised diagnostic \oraclelda{} reads a
perfect $1.000$ on all five seeds (Table~\ref{tab:leak}; per-seed values in
Appendix~\ref{app:perseed}). A perfectly
decodable signal that every unsupervised detector reads at or below chance (one of them inverted) is contradictory unless the ID and ``OOD'' populations are
reversed.

\subsection{The Cause}
They are. The class designated ``OOD'' (Memo) is one of the classes the model was
trained on, and Memo's training documents are inside the one-class
fit set. A one-class detector scores points by distance from the ID bulk; a
trained class sits squarely inside that bulk, so the detector ranks it as
more in-distribution than the held-out ID evaluation data, exactly the
observed reversal (Figure~\ref{fig:leak}). This is a property of how this near-OOD split was constructed,
not of the underlying method: the attention-masking work we build on
\citep{base} constructs its intra-dataset OOD correctly, holding the designated
class (Advertisement) out of training. The leak appears only when a trained class
is repurposed as ``OOD'' without removing it from the fit set, an easy mistake
when one wants a hard, semantically close near-OOD pair. Our baselines
reconcile to the prior audit \citep{audit} to $10^{-4}$, so the effect is real, not an
implementation bug.

\subsection{The Correction}
We remove the contamination two ways. (a) Zero-retraining: substitute a
class never present in training (Advertisement, the very class the original method
paper \citep{base} holds out); with the same model and ID fit set and only the
designated ``OOD'' class swapped to this genuinely held-out one, \prs-Maha rises
from $0.326$ to $0.936$. (b) Leave-class-out (gold standard):
re-fine-tune $20$ LayoutLMv3 models ($4$ holdout classes $\times\,5$ seeds), each
with the holdout class deleted from training (eval-acc $0.93$--$0.98$), so the
``OOD'' class is genuinely unseen. Table~\ref{tab:leak} reports the corrected
result; the catastrophic inversion is gone on every holdout, and the same
Memo class that scored $0.326$ now scores $0.911$. We replicate the entire protocol
in a second domain ($15$ further RoBERTa fine-tunes on 20\,Newsgroups;
Table~\ref{tab:summary}, Section~\ref{sec:gap}), confirming the leak and its correction
are not specific to documents.

\begin{table}[htbp]

\centering\small
\begin{tabular}{@{}lcccc@{}}
\toprule
Holdout (Memo) & \prs-Maha & Mahalanobis & $k$NN & \oraclelda{} \\
\midrule
Leaked construction (Memo trained, in fit set)  & \textbf{0.326} & 0.555 & 0.634 & 1.000 \\
\textbf{Corrected} (Memo deleted, model re-fit)        & \textbf{0.911} & 0.937 & 0.924 & 0.985 \\
\bottomrule
\end{tabular}
\\[4pt]
\begin{tabular}{@{}lcccc@{}}
\toprule
Other leave-class-out holdouts & \prs-Maha & Mahalanobis & $k$NN & \oraclelda{} \\
\midrule
Resume     & 0.888 & 0.880 & 0.883 & 0.992 \\
Form       & 0.891 & 0.914 & 0.902 & 0.967 \\
Scientific & 0.818 & 0.862 & 0.849 & 0.959 \\
\bottomrule
\end{tabular}
\caption{The leak and its correction (5-seed mean \auroc). \textbf{Top:} the same
Memo class shifts from below chance ($0.326$) to a normal score ($0.911$)
once it is removed from training and the fit set, while \oraclelda{} stays high in
both rows, the diagnostic contrast formalized in Remark~\ref{rem:fingerprint}.
\textbf{Bottom:} three additional leave-class-out holdouts behave normally under the
corrected construction (none is a trained class). \prs{} still does not beat the
raw-embedding baseline, a separate point taken up in Section~\ref{sec:gap}.}
\label{tab:leak}
\end{table}

\subsection{A Test for the Leak}
The contradiction that exposed the leak is itself a usable diagnostic, and it
follows from the structure of one-class detection. The setting is a post-hoc
audit: given a benchmark whose OOD split is already designated, we ask
whether that split is contaminated, using the designation it ships with and a
lightweight LDA on cached features, with no re-training of the backbone. A trained
class is, by construction, perfectly linearly separable from the rest of ID (so a
supervised reader achieves $\oraclelda\!\to\!1$), yet it lies inside the ID fit set
(so an ID-only detector ranks it as in-distribution and scores it well below an
honest baseline, at or below chance in the strongest cases). A genuinely novel near-OOD class
need not be perfectly separable. The two regimes are therefore distinguishable
without any backbone re-training: the leaked Tobacco setting is the only one of our
$20$ settings (the $19$ clean settings plus the leaked one) with
$\oraclelda{}=1.000$ on all five seeds and
best-unsupervised \auroc{}~$<0.65$, whereas a genuinely hard novel class
(\texttt{comp.sys.mac.hardware} in our second domain, Table~\ref{tab:summary}) has
every unsupervised detector near chance ($0.51$--$0.60$) but \oraclelda{} only
$0.882$, not perfect.

\begin{remark}[Leak fingerprint]
\label{rem:fingerprint}
Near-perfect supervised decodability ($\oraclelda\!\approx\!1$) together with
unsupervised detection that collapses well below its clean and far-OOD level
(operationally, best-unsupervised $\auroc<0.65$, and below chance in the strongest
cases) is a symptom of a trained class leaking into the in-distribution fit set, and
can flag a suspect near-OOD split before any backbone re-training. Section~\ref{sec:validate} validates this diagnostic on a
controlled battery (in embedding space: sensitivity $18/20$, specificity $31/32$ over
four fine-tuned backbones), and Section~\ref{sec:audit} runs the audit it enables across
$24$ standard published OOD benchmark pairs.
\end{remark}

\section{A Controlled Validation of the Fingerprint}
\label{sec:validate}
Remark~\ref{rem:fingerprint} distilled the fingerprint from a single real leak. We now
test it as a sensitive and specific diagnostic on a controlled battery in which
the ground truth is known by construction: we deliberately leak trained classes into the
detector's fit set (positives) and build matched clean controls (negatives), and ask
whether the fingerprint fires exactly on the contaminated settings and nowhere else.

\subsection{Design}
Two architectures, ResNet-50~\citep{resnet} (a CNN) and ViT-B/16~\citep{vit} (a
transformer), ImageNet-pretrained and fine-tuned, on two base datasets, CIFAR-10 and
CIFAR-100~\citep{cifar}, the full $2\times2$ of four fine-tuned backbones (ResNet-50 and
ViT-B/16, each on CIFAR-10 and CIFAR-100). We are explicit that the $52$ settings are
not independent draws: they are class/group re-designations on these four backbones
(the CIFAR-100 leaks, for instance, are five 5-class groups per backbone sharing one model,
fit set, and covariance), so the confidence intervals below are within-design, not
population-level, and the breadth claim is ``four backbones spanning a CNN and a transformer
on two datasets,'' a full $2\times2$ rather than a large independent sample. Holding the
model fixed, we score three regimes on cached penultimate features and vary only what
the ``OOD'' set is:
\begin{itemize}\setlength{\itemsep}{1pt}
  \item \textbf{leak} ($20$ settings): the OOD class's training data is inside the
  one-class fit set (CIFAR-10: single classes; CIFAR-100: 5-class groups, $500$ OOD test
  points). The fingerprint should fire.
  \item \textbf{clean-excl} ($20$, matched): the identical model and OOD class, but
  that class is removed from the fit set. This isolates fit-set membership (the one variable a leak is about) with no retraining. Should not fire.
  \item \textbf{clean-novel} ($12$): OOD is a different source: SVHN~\citep{svhn},
  DTD~\citep{dtd}, or the other CIFAR (a hard near-OOD pair). Should not fire.
\end{itemize}
We apply the same rule as on the real leak, $\textsc{fired}=(\oraclelda\!\ge\!0.95)\wedge
(\text{best-unsup}\!<\!0.65)$, best-unsup $=\max(\text{Maha},k\text{NN})$, with a
shrinkage-LDA supervised probe on a held-out split;
three seeds vary the probe split. Full per-setting and per-seed numbers are in
Appendix~\ref{app:validate}. The diagnostic operates directly on penultimate
embeddings and needs no perturbation machinery: we audit in this space deliberately,
because it is where a held-out trained class is unsupervisedly separable, and it is
also the cheapest, most widely available representation. Appendix~\ref{app:prs77} reports
a spot-check on the original $77$-dimensional perturbation signature: the fingerprint's
firing transfers there cleanly (all $15$ leak runs fire; far-OOD does not), but
the cheap no-retrain fit-set-exclusion contrast loses near-OOD specificity in that weaker signature, itself a decodable-but-not-detectable effect
(Section~\ref{sec:gap}), which is precisely why the audit belongs in the embedding space,
not in the perturbation signature.

\subsection{Result}
The fingerprint attains sensitivity $18/20$ and specificity $31/32$, with all but
one \textsc{fired} decision invariant to the supervised-probe seed (Figure~\ref{fig:validate},
Table~\ref{tab:validate}; the single seed-unstable setting is one borderline CIFAR-100$\times$ViT
leak group). We are precise about what this invariance means: the three seeds vary only the
held-out LDA probe split, so it reflects that the settings sit far from the thresholds, not
replication across independent data or model re-draws. The decisive evidence is the
matched contrast: for the same model and class, best-unsup is $0.25$--$0.78$ when the
class is leaked into the fit set and rises to $0.81$--$1.00$ when it is excluded, while
supervised decodability stays $\approx\!1.0$ throughout, isolating fit-set membership as the
cause, exactly mirroring the real Tobacco contrast ($0.326\!\to\!0.911$). All $20$ clean-excl
controls and all eight far-OOD (SVHN/DTD) controls are correct.

\subsection{Threshold Robustness, Disaggregation, and Confidence}
The two thresholds ($\oraclelda\!\ge\!0.95$, best-unsup $<0.65$) were fixed
a priori from the single Tobacco leak, not tuned on this battery, so we check
that the decision is not knife-edge. Treating leak-vs-clean as ground truth, the
fingerprint score (oracle-gated $1-\text{best-unsup}$) separates the two classes with
AUROC $0.995$, and the \textsc{fired} decisions are stable across a wide band
of the unsupervised cutoff (Table~\ref{tab:threshold}): specificity is $0.97$ for every
$\tau\!\in\![0.65,0.75]$ and rises to $1.00$ at $\tau\!\le\!0.60$, while sensitivity is
$\ge\!0.80$ for $\tau\!\ge\!0.60$. Sweeping the oracle gate as well, the decision is
likewise stable for a gate in $[0.90,0.95]$ ($18/20$, $31/32$); tightening it further trades
sensitivity for specificity as expected ($17/20$, $32/32$ at $0.97$), since a higher
decodability bar excludes the harder leaks. Because the leak and clean populations are nearly
separable, the exact cutoff is not load-bearing. Disaggregating specificity by control
type is more informative than the pooled $31/32$: the matched, confound-free
fit-set-exclusion controls are $20/20$, far-OOD (SVHN/DTD) are
$8/8$, and the hard cross-dataset near-OOD negatives are $3/4$; the lone false positive lives entirely in the last, hardest cell. Wilson $95\%$
intervals reflect the modest $n$: sensitivity $18/20=0.90$ $[0.70,0.97]$, specificity
$31/32=0.97$ $[0.84,0.99]$, fit-set-exclusion $20/20$ $[0.84,1.0]$, hard near-OOD
$3/4=0.75$ $[0.30,0.95]$. We report these wide intervals plainly: the battery
establishes the effect and its direction across the full $2\times2$ of four fine-tuned
backbones (two architectures, two datasets), while precise rates on the hardest cell await the in-the-wild audit
(Section~\ref{sec:discuss}).

\begin{table}[htbp]

\centering\small
\begin{tabular}{@{}lccccc@{}}
\toprule
best-unsup cutoff $\tau$ & 0.55 & 0.60 & \textbf{0.65} & 0.70 & 0.75 \\
\midrule
sensitivity (leak, $n{=}20$)  & 0.65 & 0.80 & \textbf{0.90} & 0.90 & 0.90 \\
specificity (clean, $n{=}32$) & 1.00 & 1.00 & \textbf{0.97} & 0.97 & 0.97 \\
\bottomrule
\end{tabular}
\caption{Threshold robustness on the controlled battery ($52$ settings). The unsupervised
cutoff $\tau$ is varied with the oracle gate fixed at $0.95$; the paper's operating point is
$\tau{=}0.65$ (bold). Decisions are stable across $\tau\!\in\![0.65,0.75]$ and
leak-vs-clean separation is near-perfect (AUROC $0.995$), so the rule is not tuned to
a single threshold.}
\label{tab:threshold}
\end{table}

\begin{figure}[htbp]
\centering
\includegraphics[width=0.74\linewidth]{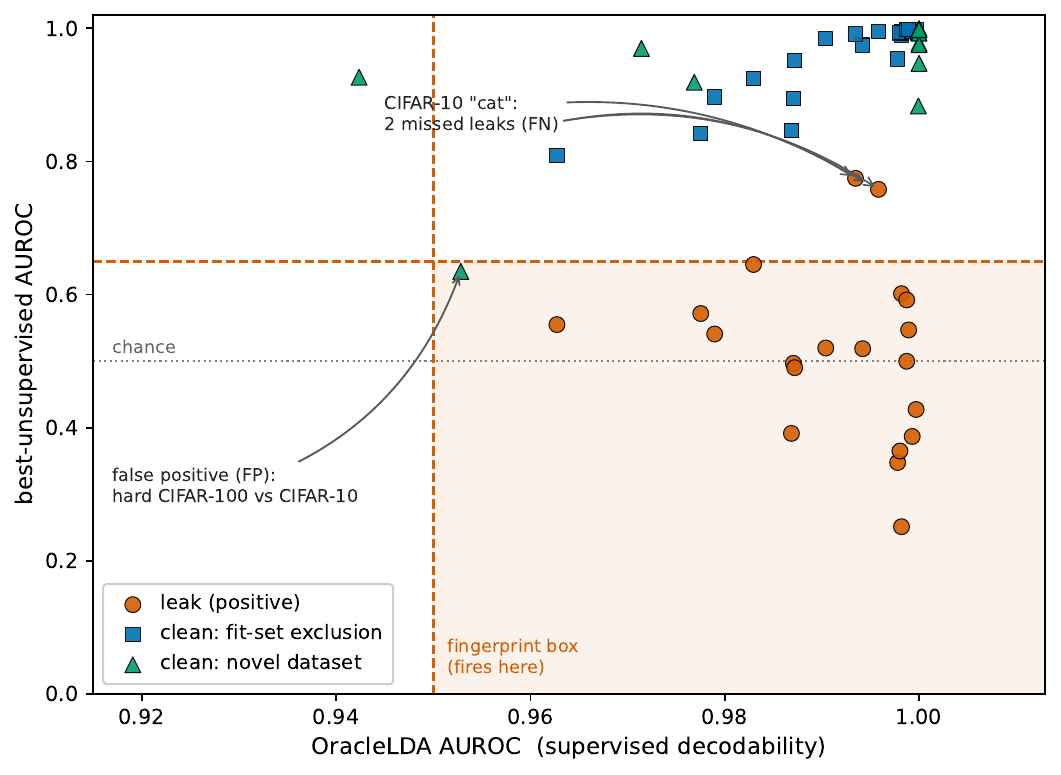}
\caption{Controlled validation of the leak fingerprint ($52$ settings, full $2\times2$).
All settings have high supervised decodability (OracleLDA $\ge 0.94$), so the horizontal axis
is cropped to the populated range. Leaked settings (orange circles) fall in the fingerprint box (high decodability with unsupervised detection collapsed below $0.65$), while both
families of clean controls (fit-set exclusion, blue squares; novel dataset, green triangles)
sit outside it. The only errors are boundary points: two leaks (CIFAR-10 ``cat'' on
both ResNet-50 and ViT-B/16, best-unsup $0.78$/$0.76$, just above the threshold) and one
clean cross-dataset pair (CIFAR-100 with CIFAR-10 as OOD) at $(0.953,0.635)$, just inside it.
Sensitivity $18/20$, specificity $31/32$.}
\label{fig:validate}
\end{figure}

\begin{table}[htbp]

\centering\small
\begin{tabular}{@{}lccccc@{}}
\toprule
Regime & $n$ & OracleLDA & best-unsup & firings & correct \\
\midrule
\textbf{leak} (should fire)        & 20 & 0.96--1.00 & 0.25--0.78 & 18 & \textbf{18/20} (sensitivity) \\
\textbf{clean-excl} (should not)   & 20 & 0.96--1.00 & 0.81--1.00 & 0  & \textbf{20/20} \\
\textbf{clean-novel} (should not)  & 12 & 0.94--1.00 & 0.64--1.00 & 1  & \textbf{11/12} \\
\midrule
\multicolumn{6}{l}{\textbf{Overall: sensitivity $18/20=0.90$, \quad specificity $31/32=0.97$.}}\\
\bottomrule
\end{tabular}
\caption{Controlled-validation summary ($52$ settings, full $2\times2$; mean over $3$ seeds,
all but one \textsc{fired} decision identical across seeds). The matched leak vs.\
clean-excl rows use the same models and OOD classes, differing only in whether the OOD
class is in the fit set.}
\label{tab:validate}
\end{table}

\subsection{Edge Cases Delineate the Diagnostic's Scope}
(i) Both false negatives are the same class, CIFAR-10 ``cat'', on both
ResNet-50 (best-unsup $0.775$) and ViT-B/16 ($0.758$): a feature-peripheral class an
unsupervised detector ranks as borderline-OOD even when its training data is in the fit set.
That the identical class misses on two independent architectures is itself informative, since it shows the miss is a property of the class rather than a fluke of one model. (ii) The lone false positive
is a CIFAR-100 model with CIFAR-10 as OOD ($\oraclelda{}=0.953$, best-unsup $=0.635$, both within $0.02$ of the thresholds). CIFAR-100$\leftrightarrow$CIFAR-10 is a textbook
intrinsically-hard near-OOD pair: CIFAR-10's broad categories are covered by CIFAR-100's
fine ones, so CIFAR-10 falls inside the CIFAR-100 feature distribution and is itself decodable but not detectable (Section~\ref{sec:gap}), with no leak at all; the
three other cross directions (including the same pair on ViT) correctly do not fire. This pins
down the scope precisely: the reliable signal of contamination is the fit-set-exclusion
contrast (perfect, $20/20$), not a fixed single-setting threshold applied blindly to an
arbitrary benchmark, because a genuinely hard clean pair can also drive unsupervised detection
down past the $0.65$ threshold (here to $0.635$) without any leak.

\section{An In-the-Wild Audit of Standard OOD Benchmarks}
\label{sec:audit}
The fingerprint's purpose is to audit existing benchmarks cheaply. We now do exactly
that, on the de-facto standard OOD-benchmark dataset pairs used across the literature (the
building blocks of OpenOOD~\citep{openood}). For each of the four fine-tuned backbones we score
every standard OOD source against it (CIFAR-10, CIFAR-100, SVHN, DTD, MNIST~\citep{mnist},
FashionMNIST~\citep{fmnist}, and STL-10~\citep{stl10}) in penultimate-embedding space under
the correct ID-train-only fit, giving $24$ near/far OOD benchmark pairs ($\times 3$
seeds; full table in Appendix~\ref{app:audit}).

The fingerprint fires on exactly one of $24$ pairs (Figure~\ref{fig:audit}): on
no far-OOD pair ($0/18$: SVHN, DTD, MNIST, FashionMNIST, STL-10 all stay well above the
threshold, best-unsup $0.75$--$1.00$), and on only one of the six semantically-near pairs, the CIFAR-100 model with CIFAR-10 as OOD ($\oraclelda{}=0.953$, best-unsup $=0.635$),
the same intrinsically-hard pair our controlled study already isolates as a
threshold-scope false positive (and it is architecture-specific: the ViT backbone on the same
pair does not fire). STL-10, which shares $9$ of CIFAR-10's $10$ classes and is a
genuine semantic near-OOD, does not fire on either architecture (best-unsup
$0.77$/$0.80$). This has two consequences. Standard cross-dataset benchmark construction
does not introduce the fit-set leak: it never places a trained class inside the one-class fit
set, so the leak we diagnose simply does not arise, and the only fire is the predicted
hard-pair scope rather than a contamination. The diagnostic is therefore specific in the wild ($23/24$ correct under a blind single-threshold rule, $0$ far-OOD false fires). The leak
we found in the Tobacco document benchmark was therefore a construction-discipline lapse the
standard image protocol avoids, and one our cheap test would have caught.

\begin{figure}[htbp]
\centering
\includegraphics[width=0.74\linewidth]{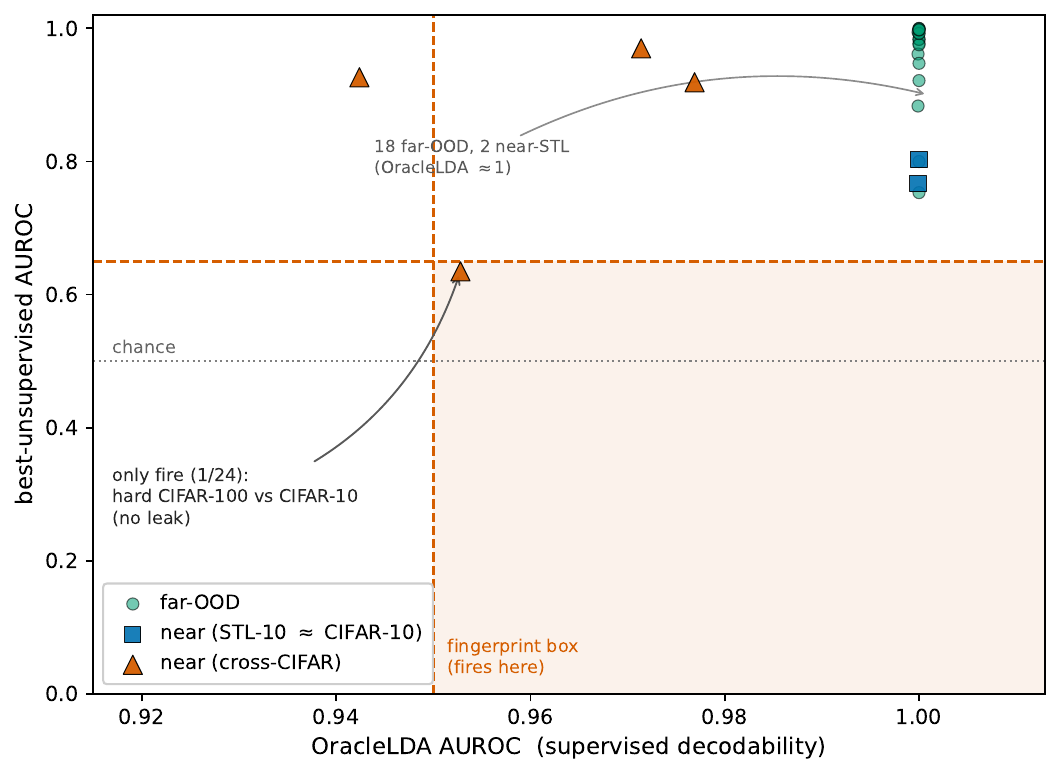}
\caption{In-the-wild audit: $24$ standard near/far OOD benchmark pairs across four backbones,
scored by the fingerprint in embedding space (the horizontal axis is cropped to the populated
range, OracleLDA $\ge 0.94$; the $20$ far-OOD and near-STL pairs at OracleLDA $\approx 1$ overlap
near $x{=}1$). Every far-OOD pair (and every semantically-near pair but one) sits outside the
fingerprint box. The single fire is the intrinsically-hard CIFAR-100$\to$CIFAR-10 pair
(the documented threshold-scope false positive, no leak); STL-10, which shares $9/10$ CIFAR-10
classes, correctly does not fire. Standard cross-dataset construction does not introduce the
fit-set leak ($23/24$ correct, $0/18$ far-OOD false fires).}
\label{fig:audit}
\end{figure}

We repeat the audit in the leak's native modalities, applying the same embedding-space
fingerprint to $12$ cached RoBERTa (text) and LayoutLMv3 (document) near/far-OOD settings
(Appendix~\ref{app:crossmodal}). It yields zero false positives: all $11$ clean settings (the held-out-Advertisement document split, four document and three text leave-class-out
holdouts, and the SST-2$\to$\{20news, AGNews, MNLI\} cross-task pairs) stay outside the
fingerprint box, so specificity in the wild holds across all three modalities. One case is
especially telling: the genuinely-hard \texttt{20news/mac.hw} holdout has low
unsupervised separability (best-unsup $0.57$) yet is correctly not flagged, because its
supervised decodability ($\oraclelda{}=0.91<0.95$) fails the oracle gate (the gate distinguishing a hard novel class from a leaked trained one), precisely the
discrimination a purely unsupervised overlap index cannot make (Section~\ref{sec:related},
Appendix~\ref{app:headtohead}). The one genuinely-leaked setting, Tobacco/Memo, sits at the
decision boundary in embedding space (mean best-unsup $0.64$, firing on $2$ of $5$ seeds) and is
caught decisively only in its native $77$-d perturbation signature (best-unsup $0.33$;
Section~\ref{sec:leak}), consistent with the representation-dependence documented in Appendix~\ref{app:prs77}, which is why embedding space is the right place to screen for
specificity while the leak signal itself is sharpest in the detector's own
representation.

\section{A Clean Re-evaluation: Decodable but Not Detectable}
\label{sec:gap}
Correcting the protocol removes the artifact and lets us state what is actually
true about perturbation-response detection. Across all $19$ clean settings
(a representative subset in Figure~\ref{fig:gap} and Table~\ref{tab:summary}; full
per-setting numbers released with the code), three architectures (ViT-B/16, RoBERTa, LayoutLMv3), seven datasets, the
leave-class-out holdouts in two domains, and an activation-noise variant, the picture is uniform. The perturbation detector provides no useful improvement
over a plain Mahalanobis-on-embeddings baseline: of the twelve settings in
Table~\ref{tab:summary}, \prs-Maha falls below plain Mahalanobis on ten and exceeds
it only on LOO Resume (by $0.008$) and on the uniformly hard \texttt{mac.hw} (at a
near-chance $0.595$, where every unsupervised score is poor); it is also beaten on the
clean far gate by a single projected dimension ($0.880$ vs $0.868$) and frequently by
an off-the-shelf Isolation Forest on its own signatures ($0.916$ vs $0.868$), and in
no setting does it approach the supervised \oraclelda. The signal is nonetheless there: the
supervised \oraclelda{} reads it at $0.87$--$1.00$, leaving a persistent gap
$\Gamma=\oraclelda-\max(\text{unsup.})$ that runs from $\approx 0$ on clean far-OOD
(CIFAR$\to$SVHN: $0.001$), is modest ($0.02$--$0.10$) on most near-OOD, and reaches
$\approx 0.29$ on the hardest clean setting (\texttt{mac.hw}; Table~\ref{tab:summary}).
The signal is decodable with supervision but not detectable without it.

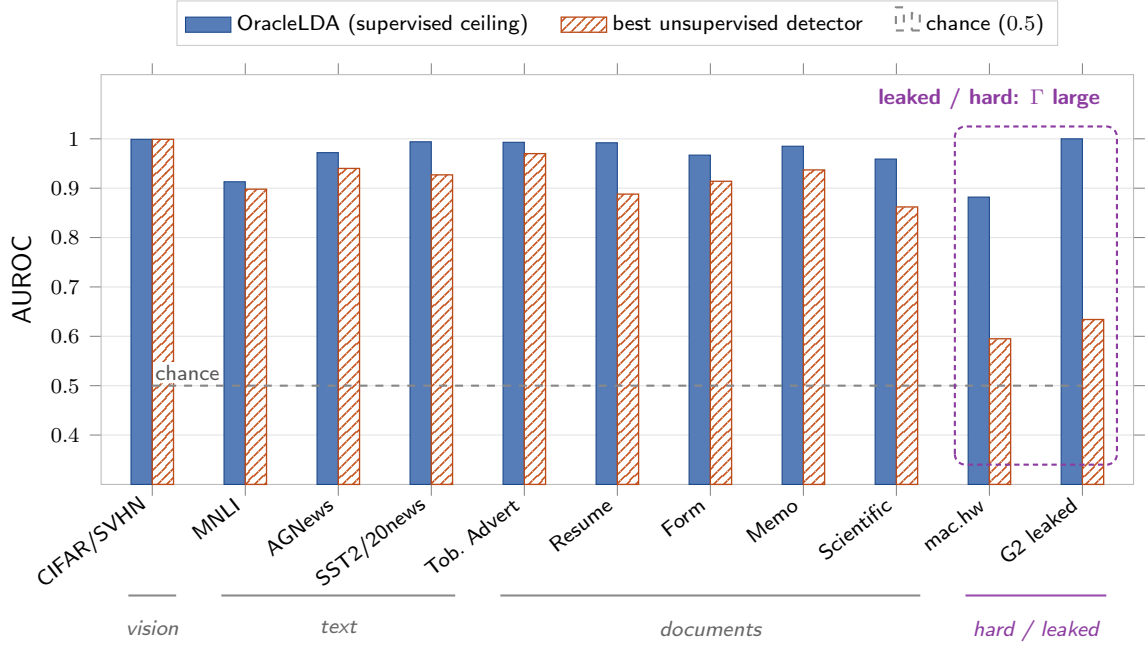
\begin{figure}[htbp]
\centering
\begin{tikzpicture}
\begin{axis}[
    width=\linewidth, height=7.0cm,
    ybar=0pt, bar width=8pt,
    enlarge x limits=0.055,
    ymin=0.3, ymax=1.13,
    ytick={0.4,0.5,0.6,0.7,0.8,0.9,1.0},
    ymajorgrids, grid style={gray!22},
    ylabel={AUROC}, ylabel style={font=\sffamily\small},
    symbolic x coords={CIFAR/SVHN,MNLI,AGNews,SST2/20news,Tob.\ Advert,Resume,Form,Memo,Scientific,mac.hw,G2 leaked},
    xtick=data,
    x tick label style={rotate=38, anchor=east, font=\sffamily\scriptsize},
    y tick label style={font=\sffamily\scriptsize},
    legend style={at={(0.5,1.07)}, anchor=south, legend columns=3,
                  font=\sffamily\scriptsize, draw=gray!40, fill=white,
                  /tikz/every even column/.append style={column sep=10pt}},
    legend cell align=left, axis line style={gray!55}, tick align=outside,
    clip=false,
]
\addplot[draw=oracleblue!85!black, fill=oracleblue!75,
         area legend] coordinates {
 (CIFAR/SVHN,0.999) (MNLI,0.913) (AGNews,0.972) (SST2/20news,0.994)
 (Tob.\ Advert,0.993) (Resume,0.992) (Form,0.967) (Memo,0.985)
 (Scientific,0.959) (mac.hw,0.882) (G2 leaked,1.000)};
\addlegendentry{OracleLDA (supervised ceiling)}
\addplot[draw=unsupred!85!black, line width=0.5pt,
         pattern={Lines[angle=45,distance=2.6pt,line width=0.6pt]},
         pattern color=unsupred!90!black,
         area legend] coordinates {
 (CIFAR/SVHN,0.999) (MNLI,0.898) (AGNews,0.940) (SST2/20news,0.927)
 (Tob.\ Advert,0.970) (Resume,0.888) (Form,0.914) (Memo,0.937)
 (Scientific,0.862) (mac.hw,0.595) (G2 leaked,0.634)};
\addlegendentry{best unsupervised detector}
\draw[dashed, gapgray!140, thick] (axis cs:CIFAR/SVHN,0.5) -- (axis cs:G2 leaked,0.5);
\node[font=\sffamily\scriptsize, gapgray!200, anchor=south west, fill=white, inner sep=1pt]
   at (axis cs:CIFAR/SVHN,0.505) {chance};
\addlegendimage{dashed, gapgray!140, thick}
\addlegendentry{chance ($0.5$)}
\draw[leakpurple, thick, rounded corners, dash pattern=on 2pt off 1.5pt]
   ([xshift=-13pt]axis cs:mac.hw,0.34) rectangle ([xshift=13pt]axis cs:G2 leaked,1.025);
\node[leakpurple, font=\sffamily\bfseries\scriptsize, align=center, anchor=south]
   at (axis cs:mac.hw,1.04) {leaked / hard: $\Gamma$ large};
\def\grpy{0.075}   
\def\grpl{0.045}   
\draw[gray!90, thick] ([xshift=-9pt]axis cs:CIFAR/SVHN,\grpy) -- ([xshift=9pt]axis cs:CIFAR/SVHN,\grpy);
\node[gray!120, font=\sffamily\scriptsize\itshape, anchor=north] at (axis cs:CIFAR/SVHN,\grpl) {vision};
\draw[gray!90, thick] ([xshift=-9pt]axis cs:MNLI,\grpy) -- ([xshift=9pt]axis cs:SST2/20news,\grpy);
\node[gray!120, font=\sffamily\scriptsize\itshape, anchor=north] at (axis cs:AGNews,\grpl) {text};
\draw[gray!90, thick] ([xshift=-9pt]axis cs:Tob.\ Advert,\grpy) -- ([xshift=9pt]axis cs:Scientific,\grpy);
\node[gray!120, font=\sffamily\scriptsize\itshape, anchor=north] at (axis cs:Form,\grpl) {documents};
\draw[leakpurple!90, thick] ([xshift=-9pt]axis cs:mac.hw,\grpy) -- ([xshift=9pt]axis cs:G2 leaked,\grpy);
\node[leakpurple, font=\sffamily\scriptsize\itshape, anchor=north]
   at ($(axis cs:mac.hw,\grpl)!0.5!(axis cs:G2 leaked,\grpl)$) {hard / leaked};
\end{axis}
\end{tikzpicture}
\caption{Decodable but not detectable. For each setting, the supervised reader
(OracleLDA, solid blue) recovers the OOD signal nearly perfectly, while the best
unsupervised detector (hatched orange) lags; the gap $\Gamma$ between them is signal that
is present but unreadable without labels. The gap is near-zero for easy far-OOD
(e.g.\ CIFAR/SVHN) and modest across the genuine near-OOD settings (typically
$0.02$ to $0.10$); it grows large only in the highlighted region, where the
genuinely hard novel class (\texttt{mac.hw}) and the leaked benchmark (\texttt{G2},
shown for contrast, not in the clean set) drive the unsupervised score toward
or below chance while decodability stays high (Section~\ref{sec:gap}). We do not claim
the gap rises monotonically with OOD closeness; see the retraction in
Section~\ref{sec:theory}.}
\label{fig:gap}
\end{figure}

\begin{table}[htbp]

\centering\small
\begin{tabular}{@{}lccccc c@{}}
\toprule
Setting & \prs-Maha & Maha & $k$NN & \oraclelda{} & $\Gamma$ & cond$(\Sig)$ \\
\midrule
G1 SST-2$\to$20news (far)     & 0.868 & \textbf{0.927} & 0.924 & 0.994 & 0.067 & 1996 \\
SST-2$\to$MNLI                & 0.837 & \textbf{0.898} & 0.883 & 0.913 & 0.015 & 1996 \\
SST-2$\to$AGNews              & 0.863 & \textbf{0.940} & 0.928 & 0.972 & 0.033 & 1996 \\
G3 CIFAR-10$\to$SVHN (far)    & 0.980 & \textbf{0.999} & 0.994 & 0.999 & 0.001 & 1436 \\
Tobacco Advert.\ (clean)      & 0.936 & \textbf{0.970} & 0.961 & 0.993 & 0.023 & 596 \\
LOO Memo                      & 0.911 & \textbf{0.937} & 0.924 & 0.985 & 0.048 & 769 \\
LOO Resume                    & \textbf{0.888} & 0.880 & 0.883 & 0.992 & 0.104 & 600 \\
LOO Form                      & 0.891 & \textbf{0.914} & 0.902 & 0.967 & 0.053 & 531 \\
LOO Scientific                & 0.818 & \textbf{0.862} & 0.849 & 0.959 & 0.097 & 670 \\
20news \texttt{mac.hw} (mid)  & \textbf{0.595} & 0.514 & 0.571 & 0.882 & 0.287 & 1796 \\
act-noise 20news              & 0.838 & \textbf{0.927} & 0.924 & 0.965 & 0.038 & 2883 \\
act-noise AGNews              & 0.799 & \textbf{0.940} & 0.928 & 0.918 & $-0.021$ & 2883 \\
\midrule
G2 Memo (leaked)       & 0.326 & 0.555 & 0.634 & 1.000 & 0.366 & 596 \\
\bottomrule
\end{tabular}
\caption{Clean re-evaluation, 5-seed mean \auroc{} ($\Gamma=\oraclelda-\max(\text{\prs-Maha},\text{Maha},k\text{NN})$ over the
three unsupervised detectors shown; best unsupervised in bold). \prs-Maha tops its
row only twice (Resume, mac.hw), both by a negligible margin or in a near-chance
regime. The signal is decodable (\oraclelda{} high) but not detectable (every unsupervised
method below it). The leaked G2 row (italic, excluded from the clean set) is shown
once for contrast: its $\Gamma=0.37$ is the fingerprint of
Remark~\ref{rem:fingerprint}. Abbreviations: LOO $=$ leave-class-out;
act-noise $=$ Gaussian activation-noise perturbation instead of head masking.}
\label{tab:summary}
\end{table}

The pattern is not an artifact of the perturbation operator, the detector, or the
signature design. Replacing head masking with Gaussian activation noise
reproduces the decodability gap almost exactly: the oracle-minus-\prs{} gap is
$0.127$ under activation noise versus $0.126$ under head masking on 20news (this is
the \oraclelda{}$-$\prs{} gap, larger than the $\Gamma$ column of
Table~\ref{tab:summary}, which measures \oraclelda{} minus the best of the
three unsupervised detectors shown there). Beyond
Mahalanobis, OC-SVM, Isolation Forest, and LOF all fail to reach \oraclelda{}.
Within the signature, the displacement block, the literal perturbation
response, is the weakest of the four blocks in every setting, so what little
\prs{} achieves comes from centroid/norm geometry a plain distance already exploits.
Full tables are in Appendix~\ref{app:tables}.

\section{Why the Signal Is Decodable but Not Detectable}
\label{sec:theory}
The gap follows from two short results; proofs are in Appendix~\ref{app:proofs}. Let ID
and OOD signatures have means $\mu_0,\mu_1$, shared covariance $\Sig$, and
$\Delta=\mu_1-\mu_0$.

\begin{proposition}[An ID-only detector cannot use the discriminating direction]
\label{prop:sign}
For the one-class Gaussian (Mahalanobis) detector fit on the ID law $P_0$ alone, the
\auroc{} against an OOD mean-shift $\Delta$ (shared $\Sig$) depends on $\Delta$ only
through $\norm{\Sig^{-1/2}\Delta}$ (a distance from the ID bulk), never through its
orientation; the same holds for any scorer whose level sets respect the ID geometry
(e.g.\ $k$NN density). More fundamentally, an ID-only scorer never observes $\mu_1$,
so it cannot select the discriminating direction or attain the sign-agnostic ceiling
$\max(\auroc,1-\auroc)$ that an oracle reaches by flipping a labeled direction.
\end{proposition}

This is exactly why the leak inverts the score (the trained class is near
the ID bulk, so $\norm{\Sig^{-1/2}\Delta}$ is small) while \oraclelda{}, which sees
labels, separates it perfectly, and it is why, even on clean data, a persistent
$\Gamma$ remains. A second result explains why the perturbation detector also loses
to a plain distance even on the clean settings. Diagonalize
$\Sig=\sum_i\lambda_i v_iv_i^\top$ and let the signal mass be
$s_i=(\Delta^\top v_i)^2/\norm{\Delta}^2$. The population separation decomposes as
\begin{equation}
\Delta^\top\Sig^{-1}\Delta = \norm{\Delta}^2\!\sum_i \frac{s_i}{\lambda_i}
\qquad\text{(Mahalanobis)}\quad\text{vs.}\quad
\norm{\Delta}^2\!\sum_i s_i\ \ \text{(Euclidean)} .
\label{eq:budget}
\end{equation}

\begin{observation}[Whitening misallocates its budget to noise]
\label{prop:budget}
If the signal mass concentrates in high-variance directions (large $s_i$ where
$\lambda_i$ is large), which we verify empirically (Appendix~\ref{app:tables}, mean
top-decile mass $0.86$--$0.94$), then Mahalanobis whitening places strictly more of
its separation budget on the low-variance directions than Euclidean does, and with
$\hat\Sig$ estimated from finite samples the realized score variance is dominated by
the smallest eigenvalues ($\propto\sum_i\lambda_i^{-2}$). When the signal is absent
there, whitening adds variance without separation, lowering \auroc.
\end{observation}

This is a textbook consequence of the bias--variance behavior of whitening under
ill-conditioning rather than a new result; we make it explicit only because it
fully accounts for our negative finding. Empirically the signatures are
ill-conditioned (cond$(\Sig)$ up to $\sim\!2900$),
Mahalanobis spends $9$--$27\%$ of its budget on near-singular directions carrying
$\le 0.5\%$ of signal (mean excess $+0.185$ over Euclidean across these four
settings), and a shrinkage sweep
from full whitening toward Euclidean improves \auroc{} on most settings, by up to
$+0.08$ (largest on the ill-conditioned activation-noise variant), with a minority of
settings mildly preferring whitening (Appendix~\ref{app:tables}). Shrinking toward Euclidean when a
covariance is ill-conditioned is, of course, established practice
\citep{ledoitwolf}; our only addition is the explicit budget accounting that ties
the perturbation detector's failure to the eigenvalue spectrum of its own
signatures. Concurrent work \citep{wrongquestions} reports the same directional
effect at ImageNet scale, where interpolating a feature covariance from the
empirical estimate toward the identity improves OOD detection, and
\citet{janiak} independently find that Mahalanobis OOD detection is governed by
feature-space geometry, improving it by a radial rescaling that preserves feature
directions; both are independent confirmation of the budget mechanism we derive.

An earlier version of this work reported that $\Gamma$ grows as OOD nears
(Pearson $-0.79$). We retract it. The severity axis we used was itself an
unsupervised detectability score, making the correlation circular by construction;
recomputed against a detection-independent axis (squared MMD in a frozen
non-detector encoder) it collapses to Pearson $-0.035$ (Appendix~\ref{app:tables}).
We keep only the weak, true statement of Section~\ref{sec:gap}: the gap is persistent but
modest, and near-zero for clean far-OOD. We report this because the same
discipline, independent re-measurement, is what gives us confidence in the leak
result, and it also caught a genuine wiring bug (Appendix~\ref{app:repro}).

\section{Discussion}
\label{sec:discuss}
We separate what the evidence here supports from what it does not, since the leak result and the negative perturbation result invite over-reading in opposite directions.

\subsection{What We Claim}
The evidence supports three claims: (i) a near-OOD split built by repurposing a trained class
as ``OOD'' contaminates the ID fit set, manufacturing below-chance results that a
leave-class-out construction removes ($0.326\!\to\!0.911$ on the same class, $35$
fine-tunes, two domains); (ii) a retraining-free fingerprint (near-perfect supervised
decodability with unsupervised detection collapsed well below its clean level, best-unsup $<0.65$) flags the contamination, and is
validated, in the model's embedding space, to sensitivity $18/20$ and specificity $31/32$
on the full $2\times2$ of four fine-tuned backbones (Section~\ref{sec:validate}), and an
in-the-wild audit of $24$ standard OOD benchmark pairs fires on only the single
intrinsically-hard pair our scope predicts and no far-OOD pair (Section~\ref{sec:audit});
(iii) under the correction, perturbation-response signals are decodable but not
detectable and do not meaningfully improve on plain embedding distance, with a
theoretical account of why.
\subsection{What We Do Not Claim}
We do not claim the attention-masking method or its
original evaluation \citep{base} is flawed (it holds its OOD class out of training); the leak we study is a benchmark-construction pitfall that arises when that discipline lapses; we do not propose an OOD method; we do not show the gap
grows as OOD nears (retracted); and we do not claim an exhaustive, ImageNet-scale field survey: we audit $24$ standard near/far OOD benchmark pairs (Section~\ref{sec:audit})
and find standard cross-dataset construction free of the fit-set leak, but a larger survey
across ImageNet-scale near-OOD suites is left to future work. As a single-threshold rule the
fingerprint can false-fire on an intrinsically hard clean near-OOD pair
(Section~\ref{sec:validate},~\ref{sec:audit}), where the reliable contamination signal is the
fit-set-exclusion contrast; and that cheap no-retrain contrast is reliable in the
model's penultimate-embedding space (clean-excl $20/20$), not in the weaker
perturbation signature, where genuine leave-class-out retraining is required instead
(Appendix~\ref{app:prs77}).
\subsection{Why the Audit Is Cheap}
Confirming a $k$-class suspect split by the gold-standard
leave-class-out route costs $k$ full backbone re-trainings; the fingerprint costs
none (one supervised probe and one unsupervised score on already-cached features),
and the confound-free fit-set-exclusion contrast adds only a second detector fit per suspect class, trading $\mathcal{O}(k)$ re-trainings for $\mathcal{O}(k)$ cheap detector
fits. That cost reduction is what made the $24$-pair vision audit and the cross-modal text and
document audit (Section~\ref{sec:audit}) practical; a larger, ImageNet-scale survey is the natural
next step.

\section{Conclusion}
\label{sec:conclusion}
A perturbation-based OOD detector scoring below chance led us not to a broken method
but to a broken benchmark: a near-OOD protocol that scores a model on a class it was
trained to recognize. We correct the protocol, propose and validate a fingerprint that audits a
designated split for the same contamination without backbone re-training (in embedding
space: sensitivity $18/20$, specificity $31/32$ across the full $2\times2$ of four fine-tuned
backbones), run the audit it enables across $24$ standard OOD benchmark pairs (where
it fires on only the one intrinsically-hard pair our scope predicts, confirming standard
construction is clean and the test specific in the wild), and, under
the correction, show that perturbation signatures are decodable but not detectable,
adding nothing useful over the embeddings they come from. The transferable
contributions are the corrected evaluation and the validated, field-tested leak fingerprint;
the discipline that produced them, pre-registration plus independent
re-measurement, is what caught both a real bug and a circular claim before
publication.

\acks{We thank the maintainers of the public datasets (CIFAR-10/100, SVHN, DTD, MNIST,
Fashion-MNIST, STL-10, 20\,Newsgroups, Tobacco-3482) and the pretrained checkpoints used in this
study. This work received no specific grant from any funding agency in the public, commercial, or
not-for-profit sectors. The author declares no competing interests.}

\appendix

\section{Proofs}
\label{app:proofs}
\begin{proof}[Proof of Proposition~\ref{prop:sign}]
A one-class scorer is a functional $h=\mathcal F(P_0)$ of the ID law alone, so it
does not depend on $\mu_1$. For two Gaussians with shared $\Sig$, the ID-only
Mahalanobis score of an OOD point $z=\mu_1+\varepsilon$,
$\varepsilon\!\sim\!\mathcal N(0,\Sig)$, has
$\E[s_{\mathrm{Maha}}(z)]=-(d+\Delta^\top\Sig^{-1}\Delta)$, depending on $\Delta$
only through $\norm{\Sig^{-1/2}\Delta}$, a rotational invariant of $\Delta$ in the
whitened metric, never its orientation. Any \auroc{} from such an $h$ is a function
of $\norm{\Sig^{-1/2}\Delta}$ only. The oracle quantity $\max(\auroc,1-\auroc)$
requires choosing the sign of a fixed direction using labels, information
unavailable to $\mathcal F(P_0)$.
\end{proof}
\begin{proof}[Proof of Observation~\ref{prop:budget}]
Write the per-direction budgets as probability vectors $b^E_i=s_i$ and
$b^M_i=(s_i/\lambda_i)/Z$ with $Z=\sum_j s_j/\lambda_j$. Ordering by decreasing
$\lambda$, the ratio $b^M_i/b^E_i=(1/\lambda_i)/Z$ is non-decreasing as $\lambda_i$
decreases, so $b^M$ first-order stochastically dominates $b^E$ toward small
$\lambda$; hence $\sum_{i\in\mathcal B}b^M_i\ge\sum_{i\in\mathcal B}b^E_i$ on the
bottom-$\lambda$ set $\mathcal B$, with equality iff $\lambda_i$ is constant where
$s_i>0$. For the variance claim, with $u_i=(z-\mu_0)^\top v_i\sim\mathcal
N(0,\lambda_i)$ we have $s_{\mathrm{Maha}}=-\sum_i u_i^2/\lambda_i$; replacing
$\lambda_i$ by a noisy estimate and applying the delta method gives a variance
contribution $\propto\sum_i\lambda_i^{-2}\mathrm{Var}[\hat\lambda_i]$, dominated by
the smallest $\lambda_i$. Where $s_i\approx0$ there, this adds variance without
separation; Euclidean has no $\lambda_i^{-1}$ weighting and is unaffected.
\end{proof}

\section{Per-Seed Results for the Leak}
\label{app:perseed}
\begin{table}[htbp]

\centering\small
\begin{tabular}{@{}lccccc|c@{}}
\toprule
Method & s42 & s123 & s456 & s789 & s1024 & mean \\
\midrule
PRS-Maha & 0.4100 & 0.3933 & 0.2815 & 0.3050 & 0.2403 & 0.3260\,${\scriptstyle\pm0.065}$ \\
Mahalanobis & 0.4958 & 0.5294 & 0.7192 & 0.7223 & 0.3063 & 0.5546\,${\scriptstyle\pm0.156}$ \\
$k$NN & 0.5824 & 0.6864 & 0.8286 & 0.7867 & 0.2847 & 0.6337\,${\scriptstyle\pm0.194}$ \\
\ahss{} & 0.2748 & 0.0156 & 0.0855 & 0.0121 & 0.0907 & 0.0957\,${\scriptstyle\pm0.096}$ \\
OracleLDA & 1.0000 & 1.0000 & 1.0000 & 1.0000 & 1.0000 & 1.0000\,${\scriptstyle\pm0.000}$ \\
\bottomrule
\end{tabular}
\caption{Leaked Tobacco near-OOD, per-seed \auroc. \oraclelda{} is a perfect $1.000$
on every seed while every unsupervised score is depressed far below it (\prs-Maha and
\ahss{} below chance; Mahalanobis/$k$NN only $0.55$/$0.63$), the signature is
perfectly decodable, but the protocol leak makes the one-class detectors fail. This
is the empirical basis of the leak fingerprint (Remark~\ref{rem:fingerprint}).}
\label{tab:ps_leak}
\end{table}

\begin{table}[htbp]

\centering\small
\begin{tabular}{@{}lccccc|c@{}}
\toprule
Method & s42 & s123 & s456 & s789 & s1024 & mean \\
\midrule
PRS-Maha & 0.9021 & 0.8066 & 0.9357 & 0.8681 & 0.8260 & 0.8677\,${\scriptstyle\pm0.047}$ \\
Mahalanobis & 0.9280 & 0.8834 & 0.9549 & 0.9340 & 0.9339 & 0.9268\,${\scriptstyle\pm0.024}$ \\
$k$NN & 0.9044 & 0.8710 & 0.9632 & 0.9417 & 0.9391 & 0.9239\,${\scriptstyle\pm0.032}$ \\
AHSS & 0.0628 & 0.0540 & 0.4709 & 0.1119 & 0.2625 & 0.1924\,${\scriptstyle\pm0.158}$ \\
OracleLDA & 0.9988 & 0.9945 & 0.9968 & 0.9913 & 0.9875 & 0.9938\,${\scriptstyle\pm0.004}$ \\
\bottomrule
\end{tabular}
\caption{Corrected far-OOD reference (G1 SST-2$\to$20news), per-seed \auroc.
\prs-Maha loses to Mahalanobis on all five seeds. \ahss{} is inverted but its
sign-agnostic ceiling is high ($\approx0.81$), consistent with
Proposition~\ref{prop:sign}.}
\label{tab:ps_g1}
\end{table}

\section{Supporting Tables}
\label{app:tables}

\begin{table}[htbp]

\centering\small
\begin{tabular}{@{}lcccc|c@{}}
\toprule
Setting & $S1$ (disp.) & $S2$ ($+$norm) & $S3$ (centroid) & $\mathbf{S4}$ (full) & best baseline \\
\midrule
G1 SST-2$\to$20news      & 0.658 & 0.840 & 0.775 & \textbf{0.868} & Maha 0.927 \\
G3 CIFAR-10$\to$SVHN     & 0.946 & 0.972 & 0.975 & \textbf{0.980} & Maha 0.999 \\
G2b Tobacco Advert.      & 0.878 & 0.935 & 0.932 & \textbf{0.936} & Maha 0.970 \\
SST-2$\to$MNLI           & 0.692 & 0.827 & 0.773 & \textbf{0.837} & Maha 0.898 \\
SST-2$\to$AGNews         & 0.688 & 0.848 & 0.777 & \textbf{0.863} & Maha 0.940 \\
LOO Memo                 & 0.889 & 0.912 & 0.903 & \textbf{0.911} & Maha 0.937 \\
LOO Resume               & 0.872 & 0.880 & 0.893 & \textbf{0.888} & $k$NN 0.883 \\
LOO Form                 & 0.877 & 0.885 & 0.890 & \textbf{0.891} & Maha 0.914 \\
LOO Scientific           & 0.785 & 0.803 & 0.823 & \textbf{0.818} & Maha 0.862 \\
\bottomrule
\end{tabular}
\caption{Signature-block ablation (\prs-Maha, 5-seed mean \auroc). The $S1$
displacement block, the literal perturbation response, is the weakest column in
every row; and no block beats the raw baseline in any clean setting, except a
$\le\!0.01$ edge on LOO Resume (within seed-to-seed noise).}
\label{tab:ablation}
\end{table}

\begin{table}[htbp]

\centering\small
\begin{tabular}{@{}lccc|cc@{}}
\toprule
Setting & OC-SVM & Iso.\ Forest & LOF & \prs-Maha & \oraclelda{} \\
\midrule
G1 SST-2$\to$20news  & 0.799 & \textbf{0.916} & 0.595 & 0.868 & 0.994 \\
G2 Tobacco near (leak) & 0.320 & 0.218 & 0.508 & 0.326 & 1.000 \\
G2b Tobacco Advert.  & 0.943 & \textbf{0.960} & 0.894 & 0.936 & 0.993 \\
G3 CIFAR-10$\to$SVHN & 0.993 & 0.996 & 0.919 & 0.980 & 0.999 \\
LOO Memo             & 0.886 & \textbf{0.935} & 0.851 & 0.911 & 0.985 \\
SST-2$\to$AGNews     & 0.697 & \textbf{0.875} & 0.714 & 0.863 & 0.972 \\
SST-2$\to$MNLI       & 0.694 & 0.831 & 0.740 & 0.837 & 0.913 \\
\bottomrule
\end{tabular}
\caption{Off-the-shelf unsupervised detectors on the \prs{} signatures. None
approaches \oraclelda; Isolation Forest frequently beats \prs-Maha (bold),
showing that \prs{} is not even the best reader of its own signatures.}
\label{tab:detectors}
\end{table}

\begin{table}[htbp]

\centering\small
\begin{tabular}{@{}lccccccccc@{}}
\toprule
Setting & $\alpha{=}0$ & .01 & .05 & .1 & .25 & .5 & .75 & .9 & $\alpha{=}1$ \\
\midrule
G1 SST-2$\to$20news  & 0.868 & 0.867 & 0.865 & 0.865 & 0.866 & 0.873 & 0.886 & 0.898 & \textbf{0.908} \\
G3 CIFAR-10$\to$SVHN & 0.980 & 0.980 & 0.981 & 0.981 & 0.983 & 0.985 & 0.989 & 0.993 & \textbf{0.996} \\
G2b Advert.          & 0.934 & 0.935 & 0.938 & 0.940 & 0.943 & 0.946 & 0.949 & 0.952 & \textbf{0.953} \\
act-noise 20news     & 0.830 & 0.837 & 0.842 & 0.844 & 0.852 & 0.864 & 0.881 & 0.896 & \textbf{0.907} \\
LOO Form             & \textbf{0.891} & 0.891 & 0.891 & 0.890 & 0.889 & 0.888 & 0.887 & 0.885 & 0.881 \\
SST-2$\to$MNLI       & \textbf{0.838} & 0.837 & 0.834 & 0.832 & 0.828 & 0.825 & 0.824 & 0.822 & 0.808 \\
\bottomrule
\end{tabular}
\caption{Covariance-shrinkage sweep (\prs-Maha \auroc), $\alpha{=}0$ full whitening
$\to\alpha{=}1$ Euclidean (best per row in bold). Moving away from whitening helps on
most settings (by up to $+0.08$); a minority (e.g.\ Form, MNLI) mildly prefer full
whitening. This is consistent with Observation~\ref{prop:budget}.}
\label{tab:shrink}
\end{table}

\begin{table}[htbp]

\centering\small
\begin{tabular}{@{}lcccc@{}}
\toprule
Setting & cond$(\Sig)$ & sig.\ top-10\% & sig.\ bot-10\% & Maha budget bot-10\% \\
\midrule
G1 (20news)     & 2634 & 0.900 & 0.0006 & 0.089 \\
G3 (ViT)        & 1702 & 0.935 & 0.0020 & 0.166 \\
LOO Memo        & 1269 & 0.885 & 0.0033 & 0.229 \\
G2b (Advert.)   & 1244 & 0.858 & 0.0046 & 0.265 \\
\bottomrule
\end{tabular}
\caption{Mechanism (5-seed). The OOD signal lives in high-variance directions
(top-decile mass $0.86$--$0.94$); Mahalanobis nonetheless spends up to a quarter of
its budget on near-singular directions carrying $\le 0.5\%$ of signal. Mean
noise-amplification (Maha$-$Euclidean budget on the bottom decile) $=+0.185$ over
the four settings shown.}
\label{tab:mech}
\end{table}

\begin{table}[htbp]

\centering\small
\begin{tabular}{@{}llcc c@{}}
\toprule
Input pair & modality & MMD$^2$ (RBF) & $\Gamma$ & circular sev. \\
\midrule
SST-2$\to$20news     & text   & 0.560 & 0.067 & 0.927 \\
SST-2$\to$AGNews     & text   & 0.400 & 0.033 & 0.940 \\
SST-2$\to$MNLI       & text   & 0.143 & 0.015 & 0.898 \\
CIFAR-10$\to$SVHN    & vision & 0.171 & 0.001 & 0.999 \\
Tobacco Advert.      & doc    & 0.262 & 0.023 & 0.970 \\
Tobacco Memo         & doc    & 0.090 & 0.048 & 0.937 \\
Tobacco Resume       & doc    & 0.176 & 0.104 & 0.880 \\
Tobacco Scientific   & doc    & 0.096 & 0.097 & 0.862 \\
Tobacco Form         & doc    & 0.118 & 0.053 & 0.914 \\
\midrule
\multicolumn{2}{l}{Pearson($\Gamma$, severity)} & $-0.035$ & & $-0.787$ \\
\bottomrule
\end{tabular}
\caption{De-circularization of the retracted claim. On a detection-independent MMD
axis the gap--severity correlation is $\approx 0$; on the circular raw-Maha axis
(last column) it is $-0.79$. The ``gap grows as OOD nears'' claim does not survive.}
\label{tab:severity}
\end{table}

\section{Controlled-Validation: Full Per-Setting Results}
\label{app:validate}
All $52$ settings of the full $2\times2$ (mean over $3$ seeds; all but one \textsc{fired}
decision identical across seeds; the exception is the borderline \texttt{P\_c100\_vit\_g2}).
$\textsc{fired}=(\oraclelda\!\ge\!0.95)\wedge(\text{best-unsup}\!<\!0.65)$,
best-unsup $=\max(\text{Maha},k\text{NN})$. The \textsc{fired} column reports this rule's yes/no decision
for every setting (``no (miss)'' $=$ a leak that did not fire, i.e.\ a false negative; ``yes (FP)''
$=$ a clean setting that fired, i.e.\ a false positive); a plain ``no'' on the clean controls is
the intended outcome. Backbones: ImageNet-pretrained ResNet-50 /
ViT-B/16 fine-tuned $3$ epochs per base dataset (train acc $0.92$--$0.995$).

{\footnotesize\setlength{\tabcolsep}{4.5pt}
\begin{longtable}{@{}llcccc c@{}}
\caption{Per-setting controlled-validation results ($52$ settings, full $2\times2$). The
matched \texttt{P\_*} (leak) and \texttt{N\_excl\_*} (fit-set exclusion) rows share the same
model and OOD class. The only errors are \texttt{P\_c10\_rn\_3} and \texttt{P\_c10\_vit\_3}
(false negatives, the feature-peripheral ``cat'' class, on both architectures) and
\texttt{N\_c100\_rn\_cross} (false positive, the hard CIFAR-100$\to$CIFAR-10 near-OOD pair).}
\label{tab:validate_full}\\
\toprule
Setting & regime & OracleLDA & Maha & $k$NN & best-unsup & \textsc{fired} \\
\midrule
\endfirsthead
\multicolumn{7}{c}{\small\itshape Table \thetable\ (continued from previous page)}\\
\toprule
Setting & regime & OracleLDA & Maha & $k$NN & best-unsup & \textsc{fired} \\
\midrule
\endhead
\midrule
\multicolumn{7}{r}{\small\itshape continued on next page}\\
\endfoot
\bottomrule
\endlastfoot
P\_c100\_rn\_g0            & leak & 0.978 & 0.497 & 0.572 & 0.572 & yes \\
P\_c100\_rn\_g1            & leak & 0.987 & 0.497 & 0.409 & 0.497 & yes \\
P\_c100\_rn\_g2            & leak & 0.963 & 0.555 & 0.496 & 0.555 & yes \\
P\_c100\_rn\_g3            & leak & 0.979 & 0.536 & 0.540 & 0.541 & yes \\
P\_c100\_rn\_g4            & leak & 0.987 & 0.392 & 0.277 & 0.392 & yes \\
P\_c100\_vit\_g0           & leak & 0.987 & 0.473 & 0.490 & 0.490 & yes \\
P\_c100\_vit\_g1           & leak & 0.994 & 0.519 & 0.479 & 0.519 & yes \\
P\_c100\_vit\_g2           & leak & 0.983 & 0.645 & 0.635 & 0.645 & yes \\
P\_c100\_vit\_g3           & leak & 0.990 & 0.520 & 0.475 & 0.520 & yes \\
P\_c100\_vit\_g4           & leak & 0.998 & 0.348 & 0.267 & 0.348 & yes \\
P\_c10\_rn\_0              & leak & 0.998 & 0.577 & 0.601 & 0.601 & yes \\
P\_c10\_rn\_1              & leak & 0.998 & 0.251 & 0.169 & 0.251 & yes \\
P\_c10\_rn\_2              & leak & 0.999 & 0.490 & 0.500 & 0.500 & yes \\
P\_c10\_rn\_3 (cat)        & leak & 0.993 & 0.729 & 0.775 & 0.775 & no (miss) \\
P\_c10\_rn\_4              & leak & 0.998 & 0.347 & 0.365 & 0.365 & yes \\
P\_c10\_vit\_0             & leak & 1.000 & 0.415 & 0.427 & 0.427 & yes \\
P\_c10\_vit\_1             & leak & 0.999 & 0.387 & 0.355 & 0.387 & yes \\
P\_c10\_vit\_2             & leak & 0.999 & 0.592 & 0.555 & 0.592 & yes \\
P\_c10\_vit\_3 (cat)       & leak & 0.996 & 0.756 & 0.758 & 0.758 & no (miss) \\
P\_c10\_vit\_4             & leak & 0.999 & 0.458 & 0.547 & 0.547 & yes \\
\midrule
N\_excl\_c100\_rn\_g0      & clean-excl & 0.978 & 0.797 & 0.842 & 0.842 & no \\
N\_excl\_c100\_rn\_g1      & clean-excl & 0.987 & 0.895 & 0.821 & 0.895 & no \\
N\_excl\_c100\_rn\_g2      & clean-excl & 0.963 & 0.809 & 0.677 & 0.809 & no \\
N\_excl\_c100\_rn\_g3      & clean-excl & 0.979 & 0.897 & 0.860 & 0.897 & no \\
N\_excl\_c100\_rn\_g4      & clean-excl & 0.987 & 0.846 & 0.815 & 0.846 & no \\
N\_excl\_c100\_vit\_g0     & clean-excl & 0.987 & 0.944 & 0.951 & 0.951 & no \\
N\_excl\_c100\_vit\_g1     & clean-excl & 0.994 & 0.975 & 0.965 & 0.975 & no \\
N\_excl\_c100\_vit\_g2     & clean-excl & 0.983 & 0.925 & 0.908 & 0.925 & no \\
N\_excl\_c100\_vit\_g3     & clean-excl & 0.990 & 0.970 & 0.984 & 0.984 & no \\
N\_excl\_c100\_vit\_g4     & clean-excl & 0.998 & 0.946 & 0.954 & 0.954 & no \\
N\_excl\_c10\_rn\_0        & clean-excl & 0.998 & 0.991 & 0.995 & 0.995 & no \\
N\_excl\_c10\_rn\_1        & clean-excl & 0.998 & 0.990 & 0.989 & 0.990 & no \\
N\_excl\_c10\_rn\_2        & clean-excl & 0.999 & 0.988 & 0.997 & 0.997 & no \\
N\_excl\_c10\_rn\_3        & clean-excl & 0.993 & 0.978 & 0.992 & 0.992 & no \\
N\_excl\_c10\_rn\_4        & clean-excl & 0.998 & 0.973 & 0.993 & 0.993 & no \\
N\_excl\_c10\_vit\_0       & clean-excl & 1.000 & 0.993 & 0.997 & 0.997 & no \\
N\_excl\_c10\_vit\_1       & clean-excl & 0.999 & 0.991 & 0.995 & 0.995 & no \\
N\_excl\_c10\_vit\_2       & clean-excl & 0.999 & 0.994 & 0.998 & 0.998 & no \\
N\_excl\_c10\_vit\_3       & clean-excl & 0.996 & 0.989 & 0.995 & 0.995 & no \\
N\_excl\_c10\_vit\_4       & clean-excl & 0.999 & 0.981 & 0.998 & 0.998 & no \\
\midrule
N\_c100\_rn\_cross         & clean-novel (near $\to$C10)  & 0.953 & 0.589 & 0.635 & 0.635 & yes (FP) \\
N\_c100\_rn\_dtd           & clean-novel (far)            & 1.000 & 0.977 & 0.819 & 0.977 & no \\
N\_c100\_rn\_svhn          & clean-novel (far)            & 1.000 & 0.883 & 0.828 & 0.883 & no \\
N\_c100\_vit\_cross        & clean-novel (near $\to$C10)  & 0.977 & 0.919 & 0.889 & 0.919 & no \\
N\_c100\_vit\_dtd          & clean-novel (far)            & 1.000 & 0.994 & 0.977 & 0.994 & no \\
N\_c100\_vit\_svhn         & clean-novel (far)            & 1.000 & 0.947 & 0.928 & 0.947 & no \\
N\_c10\_rn\_cross          & clean-novel (near $\to$C100) & 0.942 & 0.896 & 0.926 & 0.926 & no \\
N\_c10\_rn\_dtd            & clean-novel (far)            & 1.000 & 0.969 & 0.993 & 0.993 & no \\
N\_c10\_rn\_svhn           & clean-novel (far)            & 1.000 & 0.931 & 0.975 & 0.975 & no \\
N\_c10\_vit\_cross         & clean-novel (near $\to$C100) & 0.971 & 0.970 & 0.964 & 0.970 & no \\
N\_c10\_vit\_dtd           & clean-novel (far)            & 1.000 & 1.000 & 0.996 & 1.000 & no \\
N\_c10\_vit\_svhn          & clean-novel (far)            & 1.000 & 0.998 & 0.993 & 0.998 & no \\
\end{longtable}
}

\section{Spot-Check on the Deployed $77$-d Perturbation Signature}
\label{app:prs77}
The controlled battery (Section~\ref{sec:validate}) operates on penultimate embeddings. To
test the fingerprint in the deployed representation (the $77$-dimensional
perturbation signature in which the original Tobacco leak surfaced), we re-ran the matched
leak / clean-excl design on the S4 signature of the ViT-B/16 CIFAR-10 model, reusing the v2
PRS extractor, signature, and scorers verbatim (five leaked classes, including ``cat''; three
seeds; identical rule). Table~\ref{tab:prs77} reports the result.

\subsection{What Transfers}
The fingerprint's firing carries over cleanly: all $15$ leak
runs fire (best-unsup $0.34$--$0.63$, \oraclelda{} $\approx0.99$), and the far-OOD control
(SVHN) is correctly rejected on every seed (best-unsup $\approx0.99$). Detecting a leak, and
rejecting far-OOD, therefore do not depend on the embedding space; the deployed signature
reproduces both ($5/5$ leaks, $1/1$ far-OOD).

\subsection{What Does Not, and Why}
The cheap no-retrain fit-set-exclusion control loses
near-OOD specificity in the signature space: $3$ of the $5$ matched clean-excl controls
false-fire (so sensitivity $5/5$, specificity $3/6$). The reason is intrinsic to this paper's
own thesis. A class merely removed from the fit set (without retraining) is still a
class the model was trained on, and the perturbation signature summarizes how the
representation responds to masking, not class identity; such a class is therefore not
unsupervisedly separable from the ID bulk in that signature (\oraclelda{} stays $\approx0.99$
while best-unsup sits at $0.5$--$0.7$). That is precisely decodable-but-not-detectable
(Section~\ref{sec:gap}) acting on the control. In the penultimate embedding the same class
is unsupervisedly separable, so the contrast is clean (clean-excl $20/20$,
Table~\ref{tab:validate_full}). Two consequences follow: (i) the audit should be run in the
penultimate-embedding space, where held-out trained classes are detectable, which is also
the cheapest, most available representation; and (ii) on the raw perturbation signature the
genuine leave-class-out retraining correction (the one used for the original Tobacco
contrast $0.326\!\to\!0.911$, which makes the class novel in the representation itself)
remains necessary; the cheap contrast does not substitute for it there. This sharpens, but
does not weaken, the main result: the fingerprint detects leaks in either space, and the
controlled specificity and the practical audit are properly stated in embedding space.

\begin{table}[htbp]

\centering\small
\begin{tabular}{@{}llcccc@{}}
\toprule
regime & class & firings & \oraclelda{} & best-unsup & \textsc{fired} \\
\midrule
leak       & 0       & 3/3 & 0.989 & 0.395 & yes \\
clean-excl & 0       & 3/3 & 0.989 & 0.517 & yes (FP) \\
leak       & 1       & 3/3 & 0.995 & 0.449 & yes \\
clean-excl & 1       & 2/3 & 0.995 & 0.691 & yes (FP) \\
leak       & 2       & 3/3 & 0.987 & 0.578 & yes \\
clean-excl & 2       & 0/3 & 0.987 & 0.692 & no \\
leak       & 3 (cat) & 3/3 & 0.991 & 0.605 & yes \\
clean-excl & 3 (cat) & 0/3 & 0.991 & 0.776 & no \\
leak       & 4       & 3/3 & 0.985 & 0.523 & yes \\
clean-excl & 4       & 2/3 & 0.985 & 0.632 & yes (FP) \\
\midrule
clean-novel & SVHN   & 0/3 & 1.000 & 0.992 & no \\
\midrule
\multicolumn{6}{l}{\textbf{Sensitivity $5/5$ leaks; far-OOD $1/1$; cheap no-retrain near-OOD contrast $2/5$.}}\\
\bottomrule
\end{tabular}
\caption{Fingerprint on the deployed $77$-d perturbation signature (ViT-B/16, CIFAR-10; mean
over seeds $\{42,123,456\}$; ``firings'' counts how many of the three seeds fired).
\oraclelda{} is $\approx0.99$ in every row (the signal is decodable throughout); leaks fire
on best-unsup and far-OOD does not, while the no-retrain clean-excl control false-fires only
on the three feature-peripheral classes, a decodable-but-not-detectable effect, not a leak.}
\label{tab:prs77}
\end{table}

\section{In-the-Wild Audit: Full Per-Pair Results}
\label{app:audit}
All $24$ standard OOD benchmark pairs (four backbones $\times$ the OOD sources other than the
backbone's own ID dataset; mean over three seeds $\{0,1,2\}$; all decisions seed-stable).
Fingerprint in penultimate-embedding space, ID-train-only fit, same rule
$\textsc{fired}=(\oraclelda\!\ge\!0.95)\wedge(\text{best-unsup}\!<\!0.65)$.

\begin{table}[htbp]
\centering\small
\begin{tabular}{@{}lllccc@{}}
\toprule
backbone & OOD source & proximity & OracleLDA & best-unsup & \textsc{fired} \\
\midrule
cifar100\_resnet50   & dtd           & far                    & 1.000 & 0.977 & no \\
cifar100\_resnet50   & fashionmnist  & far                    & 1.000 & 0.921 & no \\
cifar100\_resnet50   & mnist         & far                    & 1.000 & 0.983 & no \\
cifar100\_resnet50   & stl10         & far                    & 1.000 & 0.800 & no \\
cifar100\_resnet50   & svhn          & far                    & 1.000 & 0.883 & no \\
cifar100\_vit\_b16   & dtd           & far                    & 1.000 & 0.994 & no \\
cifar100\_vit\_b16   & fashionmnist  & far                    & 1.000 & 0.984 & no \\
cifar100\_vit\_b16   & mnist         & far                    & 1.000 & 0.753 & no \\
cifar100\_vit\_b16   & stl10         & far                    & 1.000 & 0.961 & no \\
cifar100\_vit\_b16   & svhn          & far                    & 1.000 & 0.947 & no \\
cifar10\_resnet50    & dtd           & far                    & 1.000 & 0.993 & no \\
cifar10\_resnet50    & fashionmnist  & far                    & 1.000 & 0.999 & no \\
cifar10\_resnet50    & mnist         & far                    & 1.000 & 1.000 & no \\
cifar10\_resnet50    & svhn          & far                    & 1.000 & 0.975 & no \\
cifar10\_vit\_b16    & dtd           & far                    & 1.000 & 1.000 & no \\
cifar10\_vit\_b16    & fashionmnist  & far                    & 1.000 & 0.992 & no \\
cifar10\_vit\_b16    & mnist         & far                    & 1.000 & 0.998 & no \\
cifar10\_vit\_b16    & svhn          & far                    & 1.000 & 0.998 & no \\
\midrule
cifar10\_resnet50    & stl10         & near (STL-10$\approx$C10) & 1.000 & 0.803 & no \\
cifar10\_vit\_b16    & stl10         & near (STL-10$\approx$C10) & 1.000 & 0.767 & no \\
cifar100\_resnet50   & cifar10       & near (cross-CIFAR)     & 0.953 & 0.635 & yes (FP) \\
cifar100\_vit\_b16   & cifar10       & near (cross-CIFAR)     & 0.977 & 0.919 & no \\
cifar10\_resnet50    & cifar100      & near (cross-CIFAR)     & 0.942 & 0.926 & no \\
cifar10\_vit\_b16    & cifar100      & near (cross-CIFAR)     & 0.971 & 0.970 & no \\
\bottomrule
\end{tabular}
\caption{In-the-wild audit. The fingerprint fires on exactly one of $24$ pairs, the hard
CIFAR-100$\to$CIFAR-10 cross pair (the documented threshold-scope false positive, no leak), and on no far-OOD pair. STL-10 (which shares $9/10$ CIFAR-10 classes) and the three other cross
directions stay outside the box, so standard cross-dataset construction does not introduce the
fit-set leak.}
\label{tab:audit_full}
\end{table}

\section{Head-to-Head with the Neighborhood Class-Leakage Index}
\label{app:headtohead}
We reimplement the concurrent k-NN class-leakage index of \citet{testingthetest},
$\ell_k(i)=\tfrac1k\sum_{j\in\mathcal N_k(i)}\mathbb 1[y_j\neq y_i]$,
under the binary ID/OOD labelling, so a designated-OOD point's leakage is the fraction of its
$k{=}50$ nearest neighbours (in the combined eval embedding) that are ID, and run it on the
same $52$-setting battery with known ground truth. As a leak-vs-clean detector it attains
AUROC $0.52$ (chance): mean leakage is $0.087$ on leaked settings versus $0.094$ on
clean ones, and its best single-threshold balanced accuracy is $0.51$. Our fingerprint attains
AUROC $0.995$ on the identical settings. The index is not defective: it faithfully measures
geometric overlap, the mechanism it was designed for in class-split anomaly detection. But our leak is a different mechanism: a fully separable trained class that merely sits
inside the one-class fit set. Such a class has low neighborhood leakage (its neighbours
are its own class), so an overlap index is blind to it, while the one-class detector scores it
as in-distribution and our fingerprint fires. The two diagnostics are complementary, not
competing, and a benchmark auditor benefits from both.

\section{Cross-Modal In-the-Wild Audit: Text and Documents}
\label{app:crossmodal}
The same embedding-space fingerprint applied to $12$ cached RoBERTa (text) and LayoutLMv3
(document) near/far-OOD settings (mean over five seeds; ID-train-only fit). All $11$ clean
settings are correctly not flagged ($0$ false positives across two further modalities). The
hard \texttt{20news/mac.hw} holdout (best-unsup $0.57$) is screened out by the oracle gate
($0.91<0.95$). The real Tobacco leak sits at the boundary in embedding space (best-unsup
$0.64$, $2/5$ seeds), and is caught decisively in its native $77$-d signature ($0.33$, Sec.~4).

\begin{table}[htbp]
\centering\small
\begin{tabular}{@{}lllccc@{}}
\toprule
setting & modality & kind & \oraclelda{} & best-unsup & \textsc{fired} \\
\midrule
\texttt{gate2\_tobacco\_near}        & document & \textbf{leak (real)}    & 0.997 & 0.639 & 2/5 seeds \\
\texttt{gate2b\_tobacco\_adve}       & document & clean (held-out)        & 0.993 & 0.976 & no \\
\texttt{leaveout\_Memo}              & document & clean (LOO)             & 0.988 & 0.944 & no \\
\texttt{leaveout\_Form}              & document & clean (LOO)             & 0.976 & 0.916 & no \\
\texttt{leaveout\_Resume}            & document & clean (LOO)             & 0.992 & 0.886 & no \\
\texttt{leaveout\_Scientific}        & document & clean (LOO)             & 0.948 & 0.864 & no \\
\texttt{gate1\_sst2$\to$20news}      & text     & clean (far)             & 1.000 & 0.935 & no \\
\texttt{sec\_sst2$\to$agnews}        & text     & clean (far)             & 1.000 & 0.945 & no \\
\texttt{sec\_sst2$\to$mnli}          & text     & clean (near)            & 0.985 & 0.902 & no \\
\texttt{leaveout20\_mac.hw}          & text     & clean (LOO, hard)       & 0.910 & 0.574 & no (gate) \\
\texttt{leaveout20\_guns}            & text     & clean (LOO)             & 0.918 & 0.804 & no \\
\texttt{leaveout20\_forsale}         & text     & clean (LOO)             & 0.939 & 0.704 & no \\
\bottomrule
\end{tabular}
\caption{Cross-modal audit. Zero false positives on the $11$ clean text/document settings.
\texttt{mac.hw} is a genuinely-hard novel class with low best-unsup ($0.57$) that the oracle
gate correctly screens out ($\oraclelda{}=0.91<0.95$), the supervised axis distinguishing a
hard novel class from a leak. The real Tobacco leak is borderline in embedding space and
decisive in its native signature.}
\label{tab:crossmodal}
\end{table}

\section{Reproducibility and Compute}
\label{app:repro}
All experiments use five fixed seeds $\{42,123,456,789,1024\}$; no statistic is fit
on OOD data (asserts $+$ a unit test); every inherited baseline reconciles to the
prior audit to $\Delta\!\le\!10^{-4}$; embeddings are content-addressed and cached
so scoring never re-runs the GPU. The pipeline is post-hoc; only the leave-class-out
studies train models ($20$ LayoutLMv3, eval-acc $0.93$--$0.98$; $15$ RoBERTa,
eval-acc $0.70$--$0.75$), reported with per-class, per-seed accuracy.
We fixed one bug: the perturbation rate $p$ initially had no effect; we traced
this to the extractor not forwarding $p$ (every run used the default). After the
fix, $p$ genuinely changes the masking, heads zeroed $7$ vs $57$ of $144$, $S1$
displacement $3.08$ vs $14.14$, and \prs-Maha moves from $0.880$ ($p{=}0.05$) to
$0.797$ ($p{=}0.4$). A clean $p$-sweep on all main settings remains future work.

Compute for each experiment follows.
\begin{itemize}
\item Controlled validation (Section~\ref{sec:validate}, Appendix~\ref{app:validate}):
ImageNet-pretrained ResNet-50 and ViT-B/16, fine-tuned $3$ epochs per base dataset on CIFAR-10/100, the full $2\times2$ of four backbones (train acc $0.92$--$0.995$, batch $256$
on a single A100-80GB under AMP); three seeds $\{0,1,2\}$ vary the supervised probe split;
penultimate features are cached and detectors run on CPU; the supervised probe is shrinkage LDA.
\item In-the-wild audit (Section~\ref{sec:audit}, Appendix~\ref{app:audit}): the same four
backbones scored against SVHN, DTD, MNIST, FashionMNIST, and STL-10 (grayscale inputs mapped
to three channels), reusing the identical fit/extract/score code; ID-train-only fit; three
seeds.
\item Head-to-head (Appendix~\ref{app:headtohead}) and cross-modal audit
(Appendix~\ref{app:crossmodal}): the neighborhood-leakage index is computed on the same battery
features; the cross-modal audit reuses the cached v2 RoBERTa/LayoutLMv3 penultimate embeddings
(five seeds), all on CPU, no GPU re-run.
\item $77$-d signature spot-check (Appendix~\ref{app:prs77}): reuses the v2 PRS ViT-B/16
extractor, $S4$ signature, and one-class scorers verbatim on the already-cached gate-3
CIFAR-10 signatures (three seeds $\{42,123,456\}$, kNN $k{=}50$); no GPU re-run.
\end{itemize}
All
configs, logs, per-seed numbers, the content-addressed cache, trained checkpoints, and the
figure-generation scripts are released with the code.

\vskip 0.2in
\bibliography{references}

\end{document}